\documentclass[journal]{IEEEtran}
\usepackage{times,amsmath,amssymb,graphicx,epsf,psfrag,epsfig,cite,color,amsthm,algorithm,enumerate}
\usepackage{subcaption}
\usepackage[noend]{algpseudocode}
\makeatletter
\def\STATE{\State}

\def\FOR{\For}
\def\ENDFOR{\EndFor}
\makeatother

\newtheorem{theorem}{Theorem}
\newtheorem{lemma}{Lemma}
\newtheorem{corollary}{Corollary}

\ifCLASSINFOpdf
\else
\fi
\begin{document}
%
\title{Memory-Constrained No-Regret Learning in Adversarial Multi-Armed Bandits}
%
%
%

\author{Xiao~Xu and
        Qing~Zhao,~\IEEEmembership{Fellow,~IEEE}
\thanks{Xiao Xu and Qing Zhao are with the School
of Electrical and Computer Engineering, Cornell University, Ithaca,
NY, 14850 USA (e-mails: \{xx243, qz16\}@cornell.edu).}
\thanks{This work was supported by the National Science Foundation under Grant CCF-1815559. An earlier version was posted on arXiv at https://arxiv.org/abs/2002.11804.}
}

\maketitle

\begin{abstract}
An adversarial multi-armed bandit problem with memory constraints is studied where the memory for storing arm statistics is only in a sublinear order of the number of arms. A hierarchical learning framework that offers a sequence of operating points on the tradeoff curve between the regret order and memory complexity is developed. Its sublinear regret orders are established under both weak regret and shifting regret notions. This work appears to be the first on memory-constrained bandit problems in the adversarial setting.
\end{abstract}


%
\IEEEpeerreviewmaketitle

\section{Introduction}
%
%
%
%
\IEEEPARstart{F}{irst} posed in \cite{thompson1933likelihood} for the application of clinical trials, the multi-armed bandit (MAB) problem has been studied under various models and across diverse application domains~\cite{zhao2019multi}. The name of the problem comes from likening an archetypical single-player online learning problem to playing a multi-armed slot machine (known as a bandit for its ability of emptying the player's pocket). Each arm, when pulled, generates rewards according to an unknown stochastic model or in an adversarial fashion. Only the reward of the chosen arm is revealed after each play. The objective of the player is an arm selection policy that maximizes the cumulative reward over $T$ plays. The bandit feedback model where an arm can only be observed after it is played induces the tradeoff between exploration (to gather information from less explored arms) and exploitation (to
maximize immediate reward by prioritizing arms with a good reward history).

Depending on the generative model of arm rewards, bandit problems can be categorized into the stochastic and the adversarial settings. In the former, rewards from successive plays of an arm obey a given, albeit unknown, stochastic model. In the latter, rewards are assigned by an adversary. Regardless of the reward models, a commonly adopted performance measure of an arm selection policy is regret, defined as the cumulative reward loss against a properly defined benchmark policy that assumes hindsight vision or certain clairvoyant knowledge about the underlying generative model of arm rewards. The difference between the regret measures in the stochastic and the adversarial settings is in the adopted benchmark policies.

A canonical model for stochastic bandits assumes that rewards from each arm are drawn i.i.d. from a fixed distribution. In this case, the benchmark policy in the regret definition is one that assumes the knowledge of the stochastic model, hence plays the arm with the greatest mean throughout the time horizon. The regret is measured in expectation taken over the random process of reward realizations induced by the arm selection policy. Representative studies include \cite{lai1985asymptotically,auer2002finite,garivier2011kl,vakili2013deterministic}.

The adversarial bandit problem, first studied in \cite{auer1995gambling}, is closely related to the problem of learning in repeated unknown games. In the game setting, a player's reward of taking a particular action (i.e., playing a particular arm) is jointly determined by the payoff function of the game and the actions taken by all opponents.  From the perspective of a single player, the reward can be viewed as assigned by an adversary aggregating the interactions with all opponents in the game~\cite{cesa2006prediction}. Connections between certain system-level objectives of the game (e.g., convergence to equilibria) and the regret performance of a single player against a collective adversary have been revealed~\cite{young2004strategic,lykouris2016learning,duvocelle2018learning}. See a recent survey on distributed leaning in multi-agent systems~\cite{xu2020distributed}.

Various benchmark policies have been considered for regret measures in the adversarial setting. In particular, weak regret is defined against a benchmark policy that plays the best (fixed) arm in terms of the cumulative reward in hindsight~\cite{auer2002nonstochastic}. The weak regret notion corresponds to the external regret in the game setting. A stronger regret notion is the shifting regret, where the benchmark policy is allowed to switch arms over time but limited by a hardness constraint on the number of~switchings. 

A policy is said to achieve no-regret learning if, for every sequence of rewards assigned by the adversary, the adopted regret measure has a sublinear growth rate with $T$. In other words, the policy offers, asymptotically as~$T\to \infty$, the same average reward as the specific benchmark adopted in the corresponding regret measure. A number of learning algorithms have been developed to achieve no-regret learning under various regret notions \cite{auer2002nonstochastic,audibert2009minimax}. It has been shown that randomization in arm selection is necessary for achieving no-regret learning \cite{bubeck2012regret}.

 \subsection{Main Results}
Memory complexity has not been considered in adversarial bandits. Existing learning policies require a memory space with size linear in the number $K$ of arms to store arm reward statistics. Such a linear order of memory complexity may render these learning policies impractical in applications involving a large action/arm space, for example, recommendation systems and dynamic routing in urban transportation and computer networks.

In this paper, we study the memory-constrained adversarial bandit problem where a learning policy is only given $M$ words of memory for storing input values and necessary variables, where $M$ is in a sublinear order of $K$. The memory constraint entails that past reward observations, except for a diminishing fraction, need to be either forgotten or summarized with certain succinct statistics. No-regret learning hence hinges on not only a balance between exploration and exploitation, but also a balance between what to remember and what to forget.

In this work, we develop a hierarchical learning framework that offers a sequence of operating points on the tradeoff curve between the regret order and memory complexity. Referred to as \emph{HLMC (Hierarchical Learning with Memory Constraints)}, the proposed learning framework partitions the arms into multi-level groups and the time horizon into multi-level epochs through a tree-structured hierarchy. The depth of the tree is chosen to trade off regret order with memory complexity: a deeper tree leads to a lower memory complexity at the price of a higher regret order. Using aggregated statistics for arm groups at all levels, the HLMC framework recursively selects arm groups (referred to as super arms) across epochs (referred to as super time steps) according to the tree hierarchy. Within each epoch, a memory-unconstrained learning policy can be employed to govern the selection of arm groups at the corresponding level. This hierarchical learning framework decouples the design issue of to-remember-or-to-forget induced by the memory constraint from the exploration-exploitation tradeoff induced by the bandit feedback. It hence provides a general framework for extending memory-unconstrained learning policies to memory-constrained settings.

We establish the regret performance and memory complexity of HLMC as a function of $D$, the depth of the adopted tree hierarchy. In particular, for $D=2$, HLMC consists of a leaf level of $K$ individual arms and a higher level of $\Theta(\sqrt{K})$ arm groups, each consisting of $\Theta(\sqrt{K})$ arms. In this case, the memory required by HLMC consists of two parts: one for storing group statistics used by the group-level selection strategy, the other for arm statistics within the selected group for arm selection. We show that the memory complexity of HLMC is~$\Theta(\sqrt{K})$. In terms of regret performance, we show that no-regret learning is achieved by HLMC under both weak regret and shifting regret when suitable memory-unconstrained policies are employed as learning routines at each level. Specifically, with a sublinear-order memory complexity of $\Theta(\sqrt{K})$, HLMC offers a weak regret of $O(T^{3/4}K^{1/4})$ and a shifting regret of $O(T^{3/4}V^{1/4}K^{1/4})$ up to logarithmic factors, where $V$ is the hardness constraint on the benchmark policy. 

In the general case with a $D$-level hierarchy ($D\ge 2$), the memory required by HLMC consists of $D$ parts for storing group statistics at all $D$ levels. We show that the memory complexity is of order $\Theta(DK^{1/D})$ with a weak regret order of $O(DT^{1-\frac{1}{2D}}K^{\frac{1}{2D}})$ up to a logarithmic factor. The tradeoff between regret order and memory complexity of HLMC is therefore quantified through the discrete depth $D$ of the adopted hierarchy, which can be designed in accordance with the size of the available memory space. At the two ends of the spectrum is $D=\lceil\log_2 K\rceil$ and $D=1$. In the former, HLMC achieves no-regret learning under the notion of weak regret with a memory complexity that is only logarithmic in $K$. In the latter, the problem degenerates to the memory-unconstrained setting, and HLMC reduces to a memory-unconstrained learning routine.

\subsection{Related Work}
There is a growing body of work on adversarial bandits, in both the canonical form \cite{auer1995gambling,auer2002nonstochastic,audibert2009minimax} and various variants arising in specific applications (see, for example, \cite{bande2019adversarial,vural2019minimax}). However, memory constraints have not been considered. Most related to this work are two recent studies on memory-constrained \emph{stochastic} bandit models \cite{liau2018stochastic,chaudhuri2019regret}. In the stochastic setting in \cite{liau2018stochastic,chaudhuri2019regret}, rewards from each arm are drawn i.i.d. from a \emph{fixed} distribution. Based on the sample sizes and the gaps in the sample mean, suboptimal arms can be identified up to a desired level of accuracy and subsequently eliminated from memory. Indeed, the key idea of the two algorithms proposed in~\cite{liau2018stochastic}~\cite{chaudhuri2019regret} is based on best arm identification techniques (see~\cite{bubeck2009pure} for examples). Specifically, the memory constraint is dealt with by exploring and comparing a subset of arms over a period of time and successively eliminating suboptimal arms.

The above learning policies for memory-constrained stochastic bandits, however, do not apply to the adversarial setting. Being deterministic, they incur linear regret orders against adversaries. This is also confirmed in our numerical studies in Sec. \ref{sec:numerical}. The fundamental difference between a memory-constrained adversarial bandit problem and its stochastic counterpart is that the best arm in hindsight of an adversarially chosen reward sequence can not be reliably inferred from partial observations. As a result, no arms can be reliably eliminated from consideration at any point in the learning horizon without causing significant regret. In the proposed HLMC, the memory constraint is dealt with by storing succinct aggregated arm statistics rather than completely forgetting certain set of arms.

Another type of memory constraint that has been studied in the MAB literature is temporal across time steps: a policy can only make decisions based on the reward outcomes of the $m$ most recent plays. This problem was first considered in \cite{robbins1956sequential} where a two-armed bandit problem with Bernoulli rewards was studied.  It was later shown in \cite{cover1968note} that there exists a policy with $m=2$ that achieves an asymptotically optimal average reward in the two-armed bandit instance. The decision process with temporal memory constraints was further modeled as a finite-state machine in \cite{cover1970two}, where the past reward history was aggregated as a finite-valued statistic. The objective considered in these studies was the asymptotic convergence of the empirical average reward. Analysis on the convergence rate or the regret order, however, was lacking. The objective of minimizing regret with temporal memory constraints was considered in \cite{lu2011making} under the full-information feedback setting (i.e., the rewards of all arms that the player could have played are revealed after every time step). A learning algorithm achieving  no-regret learning with~$O(m^K)$ states (each arm statistic can take $O(m)$ values) was developed. However, the full-information feedback setting is fundamentally different from the bandit setting studied in this paper. Moreover, the proposed learning algorithm needs to store a statistic of every arm and the total number of states is exponential in $K$. 

\section{Problem Formulation}

We consider an adversarial bandit problem with a finite arm set $\mathcal{A}=\{1,2,...,K\}$.  At each time $t=1,2,...,T$, a player chooses one arm to play. The reward $r_{i,t}\in[0,1]$ of playing an arm $i$ at time $t$ is assigned by an adversary. We assume that the adversary is \emph{oblivious}, i.e., the assignment of the reward at time $t$ is independent of the player's past actions. Equivalently, an oblivious adversary determines the sequence of reward vectors $((r_{1,t},...,r_{K,t}))_{t=1}^{T}$ ahead of time. We assume that the player can only observe the reward of the selected arm at each time.

The objective of the player is an online learning policy $\pi$ that specifies a sequential arm selection rule at each time $t$ based on the observation history. We assume that the policy can only use $M$ ($M=o(K)$ as $K\to\infty$) words of memory space to store input values and necessary parameters. We follow the memory model studied in \cite{chaudhuri2019regret} where each of the variables used by the policy takes $1$ word of memory\footnote{The number of bits in a word depends on how real numbers are stored in the memory, which is out of the scope of this paper.} and thus, a policy with memory size $M$ can only store $M$ statistics at any given time to summarize the reward history of arms.

The performance of policy $\pi$ is measured by regret, which is defined as the reward loss against the best benchmark action sequence
 $a^T=(a_1,...,a_T)$ with the greatest cumulative reward, i.e.,
\begin{align}
R_{\pi}(T)=\max_{a^T\in\mathcal{A}^T}\sum_{t=1}^{T}r_{a_t,t}-\sum_{t=1}^{T}r_{\pi_t,t},
\end{align}
where $\mathcal{A}^T$ is the set of all possible action sequences with length $T$ and $\pi_t$ is the arm selected by policy $\pi$ at time $t$. When there is no ambiguity, the notation is simplified to $R(T)$. 

As the regret $R(T)$ can be randomized due to the potential randomness of the arm selection policy $\pi$, we consider two types of \emph{no-regret learning} conditions in this paper. A policy~$\pi$ is said to achieve no-regret learning \emph{in expectation} if, for every sequence of rewards $((r_{1,t},...,r_{K,t}))_{t=1}^{T}$, the expected regret $E_{\pi}[R(T)]=o(T)$ as $T\to\infty$, where the expectation is taken over the possible randomness of $\pi$. The second condition states that a policy $\pi$ achieves no-regret learning \emph{with high probability} if, for every sequence of rewards and every given $\delta\in(0,1)$, the regret $R(T)=o(T)$ as $T\to\infty$ with probability at least $1-\delta$.

It is not difficult to see that achieving no-regret learning, either in expectation or with high probability, is impossible if the benchmark sequence is chosen arbitrarily \cite{auer2002nonstochastic}. Therefore, certain restrictions on the benchmark sequence is necessary to make the problem feasible. In this paper, we consider two types of regret notions with different restrictions on the benchmark sequence. The first regret notion is the so-called \emph{weak regret} where the benchmark sequence consists of a single arm, i.e.,
\begin{align}
R_{\textrm{w}}(T)=\max_{i\in\mathcal{A}}\sum_{t=1}^{T}r_{i,t}-\sum_{t=1}^{T}r_{\pi_t,t}.
\end{align}

A stronger regret notion is the so-called \emph{shifting regret} where the benchmark sequence is constrained by its \emph{hardness}. Specifically, the hardness of a sequence $a^T=(a_1,...,a_T)$ measures the total number of arm switchings over time, i.e.,
\begin{align}
H(a^T)\triangleq 1+\sum_{t=1}^{T-1}\mathbb{I}(a_t\neq a_{t+1}),
\end{align}
where $\mathbb{I}(\cdot)$ is the indicator function. The shifting regret with a hardness constraint $V$ is defined as
\begin{align}
R_{\textrm{s}}(T,V)=\max_{a^T:H(a^T)\le V}\sum_{t=1}^{T}r_{a_t,t}-\sum_{t=1}^{T}r_{\pi_t,t}.
\end{align}
It is clear that the shifting regret is a stronger notion than the weak regret: no-regret learning under the former implies no-regret learning under the latter, but not vice versa.

To achieve no-regret learning under various regret notions, a number of learning routines have been developed in the memory-unconstrained setting.  Representative algorithms include EXP3, EXP3.P, and EXP3.S that achieve no-regret learning under the notion of weak regret in expectation, with high probability, and under the notion of shifting regret in expectation, respectively. We summarizes the details of these algorithms in Appendix A.

\section{Hierarchical Learning with Memory Constraints}\label{sec:HLMC}
In this section, we propose a general learning structure: \emph{HLMC (Hierarchical Learning with Memory Constraints)} for the memory-constrained adversarial bandit problem. We first present the general framework of HLMC with a multi-level hierarchy on the partitions of the arms and the time horizon. Then we use a representative case with a two-level hierarchy to illustrate its details.

\subsection{A General Framework with Multi-Level Hierarchy}

The key to the balance between what to remember and what to forget induced by memory constraints is to summarize past reward observations through certain succinct statistics. This motivates partitions of the arms into tree-structured groups and the time horizon into tree-structure epochs through a $D$-level hierarchy. At every level of the hierarchy, reward observations from arms within a group during an epoch is aggregated as a single group statistic. Using these group statistics, the HLMC structure carries out a recursive learning procedure that successively selects and zooms into an arm group during every corresponding epoch according to the tree hierarchy. See Fig.~\ref{fig:HLMC} for an example of HLMC with a three-level hierarchy.

Through the design of the depth $D$ of the adopted hierarchy, HLMC achieves different operating points on the tradeoff curve between the regret order and memory complexity. Intuitively, a deeper hierarchy requires a smaller memory space for storing reward statistics, but incurs a higher regret order. See Sec. \ref{sec:multi} for detailed discussions.

It should be noted that HLMC is a general learning framework that decouples the tradeoff between what to remember and what to forget from the one between exploration and exploitation. The solution to the former is the design of the aggregated group statistics and the recursive learning structure. For the latter, different learning routines developed in the memory-unconstrained setting can be plugged in for group selection at each level with the goal of minimizing various notions of regret. 

\begin{figure}[t]
\hspace{-.9cm}
	\includegraphics[width=1.15\columnwidth]{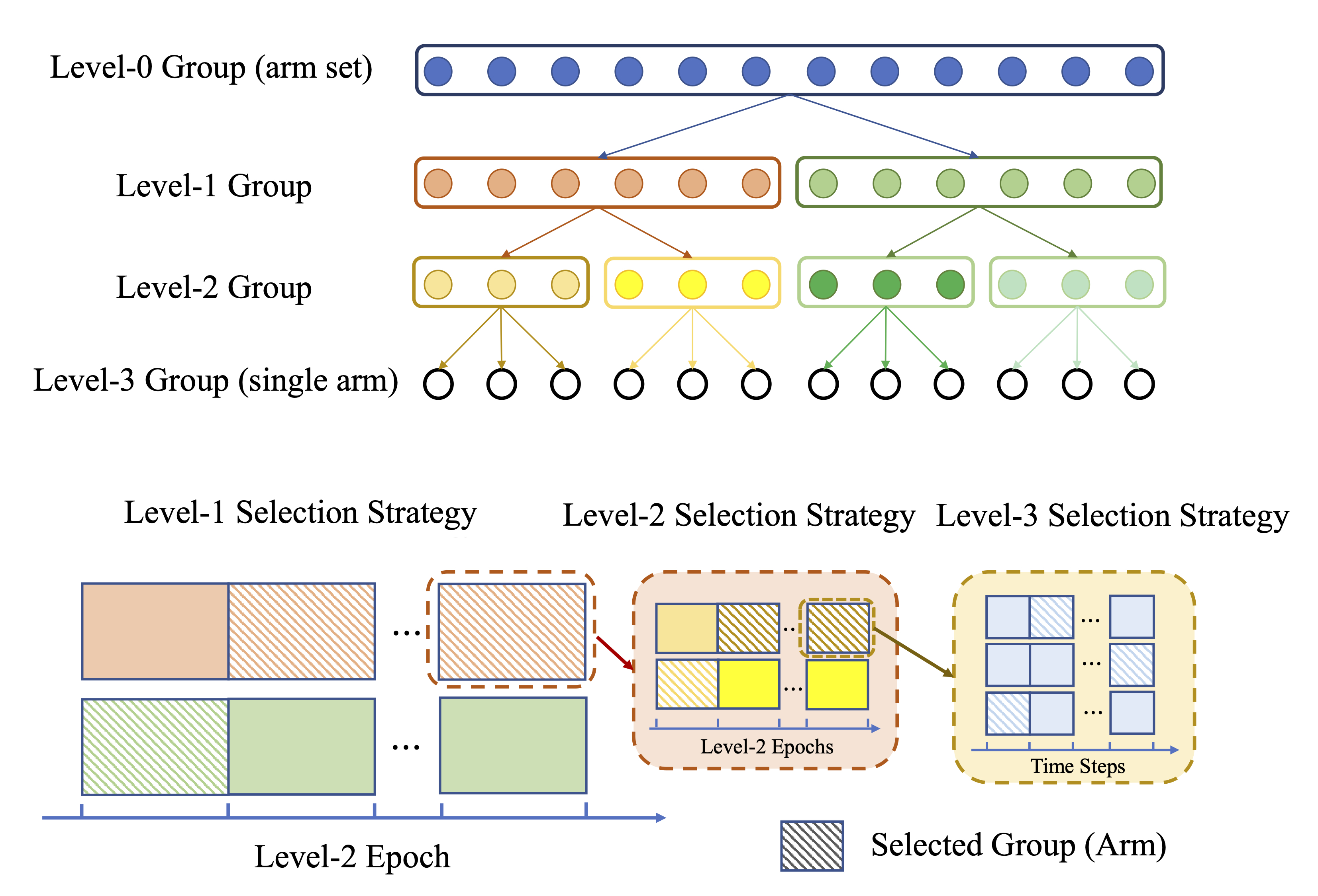}
	\caption{HLMC with a three-level hierarchy:  the arm set is partitioned into two level-1 groups and every level-1 group is partitioned into two level-2 groups (see Sec. \ref{sec:multi} for detailed discussions on the number of groups at each level). Every level-2 group consists of 3 arms (level-3 groups). The time horizon is partitioned in a similar way into multi-level epochs. At the beginning of every level-$\ell$ epoch ($1\le \ell\le 3$), a level-$\ell$ arm group is selected according to a level-$\ell$ strategy. Within this epoch, a level-$(\ell+1)$ strategy is conducted to select level-$(\ell+1)$ arm groups with in the selected level-$\ell$ group across level-$(\ell+1)$ epochs.}
	\label{fig:HLMC}

\end{figure}

\subsection{A Representative Case with Two-Level Hierarchy}\label{subsec:HLMC-2L}
We use $D=2$ as a representative case to present the details of HLMC. In the two-level hierarchy, the set $\mathcal{A}$ of arms is partitioned into equal-sized groups $\{\mathcal{A}_{\ell}\}_{\ell=1}^{L}$ where 
\begin{align}\label{eq:armpartition}
\mathcal{A}_{\ell}=\{1+N(\ell-1),...,\min(N\ell,K)\},
\end{align}
$N=\lceil\sqrt{K}\rceil$ is the group size (note that the number of arms in the last group may be smaller than $N$), and $L=\lceil\frac{K}{N}\rceil$ is the number of groups. The time horizon is partitioned into equal-length epochs $\{\mathcal{T}_s\}_{s=1}^{S}$ where
\begin{align}\label{eq:epochpartition}
\mathcal{T}_s=[1+\Delta(s-1),\min(\Delta S, T)],
\end{align}
$\Delta\in\mathbb{N}^+$ is the epoch length to be determined later, and $S=\lceil\frac{T}{\Delta}\rceil$ is the number of epochs. Note that the length of the $S$-th epoch may be smaller than $\Delta$.

By treating each group $\mathcal{A}_{\ell}$ as a ``super arm'' $\ell$ and each epoch $\mathcal{T}_s$ a ``super time-step'' $s$, we reduce the group selection problem to a classic memory-unconstrained adversarial bandit problem. Specifically, with $M\ge L$, existing learning strategies developed for memory-unconstrained adversarial bandits can be adopted using $L$ words of memory space to select groups across epochs, without violating the memory constraint. The reward of playing a ``super arm'' $\ell_s$ at a ``super time-step'' $s$ is defined as the average reward per play obtained from the corresponding arm group $\mathcal{A}_{\ell_s}$ during the corresponding epoch~$\mathcal{T}_{s}$, i.e., 
\begin{align}\label{eq:y}
y_{\ell_s,s}=\frac{1}{|\mathcal{T}_s|}\sum_{t\in\mathcal{T}_s}r_{i_t,t},
\end{align}
where $i_t\in\mathcal{A}_{\ell_s}$ is the arm selected at time $t$. 

The group-level strategy uses an aggregated statistic of every arm group, which is stored throughout the time horizon, for group selection across epochs. Once a group $\mathcal{A}_{\ell_s}$ is selected at the beginning of every epoch~$\mathcal{T}_s$, an arm-level learning routine is employed on $\mathcal{A}_{\ell_s}$ based on individual statistics of arms within the group. These arm statistics are updated at every time step in $\mathcal{T}_s$ and are forgotten at the end of the epoch. After each epoch, the average reward~$y_{\ell_s,s}$ per play is used to update the aggregated statistic of the selected group~$\mathcal{A}_{\ell_s}$. The details of HLMC with a two-level hierarchy is summarized in Algorithm~\ref{alg:HLMC}.

\begin{algorithm}[h!]
   \caption{HLMC with a Two-Level Hierarchy}
   \label{alg:HLMC}
   \begin{algorithmic}
   \STATE {\bfseries Input:} $T$ the time length, $\mathcal{A}$ the set of $K$ arms, and $\Delta>0$ the epoch length.

      \STATE Obtain arm group partition $\{\mathcal{A}_{\ell}\}_{\ell=1}^{L}$ according to (\ref{eq:armpartition}).
      \STATE Obtain epoch partition $\{\mathcal{T}_{s}\}_{s=1}^{S}$ according to (\ref{eq:epochpartition}).
      \STATE Initialize and store the statistics of every arm group.
       \FOR{$s=1,2,...,S$}
   	        \STATE {Select arm group $\ell_s$ according to the group-level}
   	        \STATE{selection strategy.}
   	        \STATE Initialize and store the statistics of every arm in $\mathcal{A}_{\ell_s}$.
   	        \STATE Initialize $y_{\ell_s,s}=0,\tau=0$.
   			\FOR{$t\in\mathcal{T}_s$}
   	        	\STATE Play arm $i_t$ according to the arm-level selection 
   	        	\STATE strategy and receive reward $r_{i_t,t}$.
   	        	\STATE Update arm statistics in the memory using $r_{i_t,t}$.
   	        	\STATE Update $y_{\ell_s,s}=\frac{y_{\ell_s,s}\tau+r_{i_t,t}}{\tau+1}$, $\tau=\tau+1$.
   	      	\ENDFOR
   	        \STATE Update all group statistics in the memory using $y_{\ell_s,s}$.
       \ENDFOR
\end{algorithmic}
\end{algorithm}

\section{Memory Complexity and Regret Performance in the Two-Level Case}\label{sec:analysis}
In this section, we analyze the memory complexity and regret performance of the proposed HLMC learning framework in the two-level case. We notice that in HLMC, the group-level strategy requires~$L$ words of memory to store a statistic of every arm group. Once a group is selected, the statistics of all arms within the selected group should also be stored. Hence,~$N$ additional words of memory are needed. As a result, the total memory size required by the HLMC framework is $N+L$, which is of order $\Theta(\sqrt{K})$.

In terms of regret performance, it is clear that the regret order in $T$ achieved by HLMC depends on the specific learning routines employed at both group and arm levels. In the following three subsections, we discuss minimizing weak regret in expectation, with high probability, and minimizing shifting regret in expectation, respectively through plugging in different learning routines to the two levels.
\subsection{Minimizing Weak Regret in Expectation}\label{subsec:weakregret}
We first show that adopting EXP3 at both group and arm levels in the HLMC framework with learning rates~$\gamma_1$ and $\gamma_2$ respectively guarantees a sublinear regret order in~$T$ under the notion of expected weak regret.

\begin{theorem}\label{thm:weakregretUB}
For any $T$ and $K$, if the input parameter $\Delta=\left\lceil\sqrt{\frac{TN\ln N}{L\ln L}}\right\rceil$ (where $N,L$ are defined in Sec. \ref{subsec:HLMC-2L}), adopting EXP3 at both group and arm levels with learning rates $\gamma_1=\sqrt{\frac{L\ln{L}}{2S}}$ and $\gamma_2=\sqrt{\frac{N\ln N}{2\Delta}}$ guarantees that, for every assignment of the reward sequence, the expected weak regret of HLMC is upper bounded by:
\begin{align}\label{eq:weakregretUB}
	\mathbb{E}_{\textrm{HLMC}}\left[R_{\textrm{w}}(T)\right]\le (4+2\sqrt{2})T^{\frac{3}{4}}K^{\frac{1}{4}}(\ln{K})^{\frac{1}{2}}.
\end{align}
\end{theorem}

To obtain the upper bound in Theorem \ref{thm:weakregretUB}, we decompose the expected weak regret into two parts by introducing an intermediate term $C'_{\max}$ as follows: for every fixed reward sequence, let~$i_{\max}$ be the best arm with the greatest cumulative reward over the entire time horizon and~$\mathcal{A}_{\ell_{\max}}$ the arm group to which $i_{\max}$ belongs. We define~$C'_{\max}$ as the expected cumulative reward obtained by running the arm-level EXP3 algorithm with learning rate $\gamma_2$ on $\mathcal{A}_{\ell_{\max}}$ during all epochs,~i.e.,
\begin{align}
C'_{\max}=\sum_{s=1}^{S}\mathbb{E}_{\textrm{Arm-EXP3}(\mathcal{A}_{\ell_{\max}})}\left[\sum_{t\in\mathcal{T}_s}r_{i_t,t}\right],
\end{align}
where $\mathbb{E}_{\textrm{Arm-EXP3}(\mathcal{A}_{\ell_{\max}})}[\cdot]$ denotes the expectation taken over the randomness of the arm-level EXP3 algorithm when conducted on group $\mathcal{A}_{\ell_{\max}}$. Then the expected weak regret of HLMC is decomposed as:
\begin{equation}\label{eq:R}
\begin{aligned}
\mathbb{E}_{\textrm{HLMC}}\left[R_{\textrm{w}}(T)\right]=\underbrace{(C'_{\max}-C_{\textrm{HLMC}})}_{R_1(T)} + \underbrace{(C_{\max}-C'_{\max})}_{R_2(T)},
\end{aligned}
\end{equation}
where
\begin{equation}
\begin{aligned}
&C_{\textrm{HLMC}}=\mathbb{E}_{\textrm{HLMC}}\left[\sum_{t=1}^{T}r_{i_t,t}\right],\\
&C_{\max}=\sum_{t=1}^{T}r_{i_{\max},t}.
\end{aligned}
\end{equation}
Note that in the decomposition, $R_1(T)$ corresponds to the group-level reward loss due to not selecting $\mathcal{A}_{\ell_{\max}}$ at every epoch, and $R_2(T)$ corresponds to the arm-level reward loss due to playing suboptimal arms in~$\mathcal{A}_{\ell_{\max}}$ assuming that group~$\mathcal{A}_{\ell_{\max}}$ is selected at all epochs.

We first upper bound the group-level reward loss $R_1(T)$. Noticing that the arm selection process during every epoch is independent of the group and arm selection history in the past, we can thus rewrite the expected reward of the HLMC policy as follows:
\begin{equation}
\begin{aligned}
&\mathbb{E}_{\textrm{HLMC}}\left[\sum_{t=1}^{T}r_{i_t,t}\right]\\
=&\mathbb{E}_{\textrm{Group-EXP3}}\left[\sum_{s=1}^{S}\mathbb{E}_{\textrm{Arm-EXP3}(\mathcal{A}_{\ell_s})}\left[\sum_{t\in\mathcal{T}_s}r_{i_t,t}\right]\right],\\
\end{aligned}
\end{equation}
where $\mathbb{E}_{\textrm{Group-EXP3}}[\cdot]$ denotes the expectation taken over the randomness of the group-level EXP3 algorithm, and $\mathcal{A}_{\ell_s}$ is the group selected at epoch $s$. To ease the analysis, we assume without losing generality that all epochs have an equal length~$\Delta$. We further define
\begin{align}
x_{\ell,s}=\mathbb{E}_{\textrm{Arm-EXP3}(\mathcal{A}_{\ell})}\left[\frac{1}{|\mathcal{T}_s|}\sum_{t\in\mathcal{T}_s}r_{i_t,t}\right].
\end{align}
It is not difficult to see that
\begin{align}\label{eq:R1}
R_1(T)=\Delta\left(\sum_{s=1}^{S}x_{\ell_{\max},s}-\mathbb{E}_{\textrm{Group-EXP3}}\left[\sum_{s=1}^{S}x_{\ell_s,s}\right]\right).
\end{align} 

It is then clear that upper bounding $R_1(T)$ is equivalent to upper bounding the weak regret of applying the group-level EXP3 algorithm to the adversarial bandit problem constructed by the reduction in Sec. \ref{sec:HLMC}. Specifically, the reward of selecting a group $\mathcal{A}_{\ell}$ at epoch $\mathcal{T}_s$ is defined as $y_{\ell,s}$ according to (\ref{eq:y}) where $i_t$ is randomly selected by the arm-level EXP3 algorithm. Therefore, $y_{\ell,s}$ is a random reward with mean $x_{\ell,s}$. The group selection problem is reduced to a classic memory-unconstrained adversarial bandit problem with noisy observations. It should be noted that after fixing an assignment of the reward sequence $((r_{1,t},...,r_{K,t}))_{t=1}^{T}$, the expected reward $x_{\ell,s}$ is fixed. Meanwhile, the realization of $y_{\ell,s}$ is independent across $\ell, s$ and is independent of the arm (group) selection history up to epoch $s$. We obtain the following result on applying the group-level EXP3 algorithm to the reduced bandit problem.

\begin{lemma}\label{lemma:groupregret}
By choosing $\gamma_1=\sqrt{\frac{L\ln L}{2S}}$, the group-level EXP3 algorithm guarantees that, for every assignment of the reward sequence $((r_{1,t},...,r_{K,t}))_{t=1}^{T}$,
\begin{equation}
\begin{aligned}
\max_{1\le\ell\le L}\sum_{s=1}^{S}x_{\ell,s}-\mathbb{E}_{\textrm{Group-EXP3}}\left[\sum_{s=1}^{S}x_{\ell_s,s}\right]\le 2\sqrt{2SL\ln L},
\end{aligned}
\end{equation}
where $\ell_s$ is the arm group selected by the group-level EXP3 algorithm at epoch $s$.
\end{lemma}
\begin{proof}
See Appendix B in the supplementary material.
\end{proof}

For the arm-level reward loss $R_2(T)$, we notice that
\begin{equation}\label{eq:R2}
\begin{aligned}
R_2(T)=\sum_{s=1}^{S}\Bigg(&\sum_{t\in\mathcal{T}_s}r_{i_{\max},t}-\mathbb{E}_{\textrm{Arm-EXP3}(\mathcal{A}_{\ell_{\max}})}\left[\sum_{t\in\mathcal{T}_s}r_{i_t,t}\right]\Bigg).
\end{aligned}
\end{equation}
It suffices to upper bound each term in the summation, that is, the weak regret of conducting the arm-level EXP3 algorithm on group $\mathcal{A}_{\ell_{\max}}$ during each epoch $\mathcal{T}_s$. The regret bound has been shown in Lemma 3.

Theorem \ref{thm:weakregretUB} is then proved by applying Lemma \ref{lemma:groupregret} and Lemma~3 to $R_1(T)$ and $R_2(T)$, respectively.

\begin{proof}[Proof of Theorem 1]

Combining (\ref{eq:R1}) with Lemma \ref{lemma:groupregret}, and~(\ref{eq:R2}) with Lemma 3, we can derive that
\begin{equation}
\begin{aligned}
R_1(T)&\le 2\Delta\sqrt{2SL\ln L}=2\sqrt{2T\Delta L\ln L},\\
R_2(T)&\le 2S\sqrt{2\Delta N\ln N}=2\sqrt{\frac{2T^2}{\Delta}N\ln N}.
\end{aligned}
\end{equation}
By choosing $\Delta=\left\lceil\sqrt{\frac{TN\ln N}{L \ln L}}\right\rceil$, we obtain the upper bound in Theorem \ref{thm:weakregretUB}.
\end{proof}

It should be noted that although the proposed learning policy requires the knowledge of the total time length $T$ for choosing input parameters to achieve no-regret learning, the issue of unknown $T$ can be easily addressed by the doubling technique as used in the classic memory-unconstrained setting~\cite{auer2002nonstochastic}. Specifically, the algorithm operates in stages, with the stage length doubles at each time. In stage $r$ with length $2^r$, the algorithm operates under a known-horizon setting with the horizon length $T=2^r$. It is not difficult to show that the same regret order still holds.

\subsection{Minimizing Weak Regret with High Probability}\label{subsec:weakregretP}
We further show that by adopting EXP3.P at both group and arm levels in the HLMC framework with parameters $(\eta_1,\gamma_1,\beta_1)$ and $(\eta_2,\gamma_2,\beta_2)$ respectively, the weak regret of HLCM has a sublinear growth rate in $T$ with high probability.

\begin{theorem}\label{thm:weakregretUBP}
For any $T,K$ and every $\delta\in(0,1)$, if $\Delta=\left\lceil\sqrt{\frac{TN\ln(2KT/\delta)}{L\ln(2L/\delta)}}\right\rceil$ (where $N,L$ are defined in Sec. \ref{subsec:HLMC-2L}), and the EXP3.P algorithm is adopted at both the group level with $\beta_1=\sqrt{\frac{\ln(2L/\delta)}{LS}},\eta_1=0.95\sqrt{\frac{\ln L}{LS}}, \gamma_1=1.05\sqrt{\frac{L\ln L}{S}}$, and the arm level with $\beta_2=\sqrt{\frac{\ln(2KS/\delta)}{N\Delta}},\eta_2=0.95\sqrt{\frac{\ln N}{N\Delta}}, \gamma_2=1.05\sqrt{\frac{N\ln N}{\Delta}}$, then for any assignment of the reward sequence, the weak regret of HLCM is upper bounded by 
\begin{equation}
R_{\textrm{w}}(T)\le 12.5T^{\frac{3}{4}}K^{\frac{1}{4}}(\ln{(2KT/\delta)})^{\frac{1}{2}},
\end{equation}
with probability at least $1-\delta$.
\end{theorem}

Theorem \ref{thm:weakregretUBP} is proved via a similar structure with that used in analyzing the expected weak regret of HLMC in Sec.~\ref{subsec:weakregret}. Specifically, the weak regret is decomposed as:
\begin{eqnarray}\nonumber
R_{\textrm{w}}(T)&=&\sum_{s=1}^{S}\sum_{t\in\mathcal{T}_s}r_{i_{\max},t}-\sum_{s=1}^{S}|\mathcal{T}_s|y_{\ell_{\max},s}\\
&&+\sum_{s=1}^{S}|\mathcal{T}_s|y_{\ell_{\max},s}-\sum_{s=1}^{S}\sum_{t\in\mathcal{T}_s}r_{i_t,t}\\ \nonumber
&=& R_1(T) + R_2(T),
\end{eqnarray}
where $i_{\max}$ is the arm with the greatest cumulative reward in hindsight, $\ell_{\max}$ is the group index of $i_{\max}$, $y_{\ell_{\max},s}$ is the average reward obtained by running the arm-level EXP3.P algorithm on $\mathcal{A}_{\ell_{\max}}$ during epoch $s$, and $i_t$ is the arm selected by HLMC at time $t$. 

We first upper bound $R_1(T)$, which corresponds to the arm-level reward loss due to playing suboptimal arms in $\mathcal{A}_{\ell_{\max}}$ assuming that $\mathcal{A}_{\ell_{\max}}$ is selected at all epochs. It suffices to upper bound 
\begin{equation}\label{hwrR1}
\sum_{t\in\mathcal{T}_s}r_{i_{\max},t}-|\mathcal{T}_s|y_{\ell_{\max},s},
\end{equation}
for every $s$. It is clear that (\ref{hwrR1}) is equivalent to the weak regret of applying the arm-level EXP3.P algorithm to $\mathcal{A}_{\ell_{\max}}$ during epoch $\mathcal{T}_s$, which is upper bounded in Lemma 4.

To upper bound $R_2(T)$, which corresponds to the group-level reward loss due to not selecting $\mathcal{A}_{\ell_{\max}}$ at all epochs, we rewrite $R_2(T)$ as
\begin{equation}
R_2(T)=\Delta\left(\sum_{s=1}^{S}y_{\ell_{\max},s}-\sum_{s=1}^{S}y_{\ell_s,s}\right)
\end{equation}
where $\ell_s$ is the group selected by the group-level EXP3.P algorithm at epoch $s$ (we assume without loss of generality that every epoch has equal length $\Delta$).

As argued in Sec. \ref{subsec:weakregret}, the realization of $y_{\ell,s}$ is independent across $\ell,s$ and is independent of the past group selection history. Once we fixed a sequence of realizations of $((y_{1,s},...y_{L,s}))_{s=1}^{S}$, Lemma 4 can be applied to upper bound the group-level regret $R_2(T)$ with high probability.

\begin{proof}[Proof of Theorem \ref{thm:weakregretUBP}]
For every $\delta>0$ and every assignment of the reward sequence, we apply Lemma 4 to all groups $\ell=1,...,L$ and all epochs $s=1,...,S$ by choosing $\delta_0=\frac{\delta}{2LS}$. Then using the union bound, we obtain that with probability at least $1-\delta/2$, the upper bound on~(\ref{hwrR1}) in Lemma 4 holds for every groups $\ell$ and every epoch $s$. As a result, the arm-level regret $R_1(T)$ is upper bounded as:
\begin{eqnarray}\label{hupR1}\nonumber
R_1(T)&\le& 5.15S\sqrt{N\Delta\ln(2NLS/\delta)}\\
&=&5.15\sqrt{\frac{T^2}{\Delta}N\ln\left(\frac{2KS}{\delta}\right)},
\end{eqnarray}
with probability at least $1-\delta/2$. 

Moreover, we apply Lemma 4 again to the group-level selection strategy by choosing $\delta_0=\delta/2$. We obtain that with probability at least $1-\delta/2$,
\begin{eqnarray}\label{hupR2}\nonumber
R_2(T)&\le& 5.15\Delta\sqrt{LS\ln(2L/\delta)}\\
&=&5.15\sqrt{T\Delta L\ln(2L/\delta)},
\end{eqnarray}
for every realization of $((y_{1,s},...y_{L,s}))_{s=1}^{S}$.
The upper bound on $R_{\textrm{w}}(T)$ in Theorem \ref{thm:weakregretUBP} is obtained by choosing $\Delta=\left\lceil\sqrt{\frac{TN\ln(2KT/\delta)}{L\ln(2L/\delta)}}\right\rceil$ and combining (\ref{hupR1}) and~(\ref{hupR2}) using the union bound.
\end{proof}

\subsection{Minimizing Shifting Regret in Expectation}\label{subsec:shiftregret}
To achieve no-regret learning under a stronger regret notion: shifting regret, we consider applying EXP3.S at the group level of HLMC. At the arm-level, we still adopt the EXP3 algorithm for arm selection. It should be noted that the arm-level strategy in the HLMC framework is restarted at the beginning of every epoch, which guarantees quick elimination of the past experience. Therefore, the hierarchical structure automatically adapts to the variation of the benchmark sequence by relying more on recent observations.  In the following theorem, we provide an upper bound on the expected shifting regret of HLMC when EXP3.S and EXP3 are adopted at the group and the arm levels, respectively.

\begin{theorem}\label{thm:shiftregretUB}
For any $T,K,$ and $V$, assume that $T\ge VK$. If the input parameter $\Delta=\left\lceil\sqrt{\frac{TN\ln N}{VL\ln (TL)}}\right\rceil$ (where $N,L$ are defined in Sec. \ref{subsec:HLMC-2L}), adopting EXP3.S at the group level with $\gamma_1=\sqrt{\frac{VL\ln{(LS)}}{S}}$, $\alpha=1/S$, and EXP3 at the arm level with $\gamma_2=\sqrt{\frac{N\ln N}{2\Delta}}$ guarantees that, for every assignment of the reward sequence, the expected shifting regret of HLMC with a hardness constraint $V$ on the benchmark action sequence is upper bounded by:
\begin{align}
	\mathbb{E}_{\textrm{HLMC}}[R_{\textrm{s}}(T,V)]\le (6\sqrt{2}+1)T^{\frac{3}{4}}V^{\frac{1}{4}}K^{\frac{1}{4}}(\ln{(KT)})^{\frac{1}{2}}.
\end{align}
\end{theorem}

\begin{corollary}
If $V=o(T)$ as $T\to\infty$, the HLMC algorithm achieves no-regret learning in expectation under the notion of shifting regret with hardness constraint $V$.
\end{corollary}

To upper bound the expected shifting regret of HLMC against an arbitrary benchmark action sequence~$a^T$ with a hardness constraint $V$, the key technique is to construct an alternative benchmark sequence $b^T$ such that: (i)~$H(b^T)\le V$, (ii) the cumulative reward achieved by~$b^T$ is close to that achieved by $a^T$, and (iii) the actions specified by $b^T$ are invariant within each epoch. Using such a sequence $b^T$, it suffices to show that the expected shifting regret of {HLMC} against $b^T$ has a sublinear growth rate in $T$.

We follow the same proof structure with that used for analyzing the expected weak regret in Sec. \ref{subsec:weakregret}. First note that the constructed sequence $b^T$ is time-invariant within each epoch. Therefore, the arm-level regret analysis in Lemma 3 directly carries over. At the group-level, the reduction to a memory-unconstrained adversarial bandit problem with noisy observations is still legitimate since the group specified by the benchmark sequence is fixed within each epoch. Based on the reduction and Lemma 5, we obtain the following result on applying the EXP3.S algorithm to the group level.

\begin{lemma}\label{lemma:groupregret-S}
By choosing $\gamma_1=\sqrt{\frac{LV\ln (LS)}{S}}$ and $\alpha=1/S$, the group-level EXP3.S algorithm guarantees that, for every assignment of the reward sequence $((r_{1,t},...,r_{K,t}))_{t=1}^{T}$ and every benchmark sequence of arm groups $h^S=(h_1,...,h_S)$ where $H(h^S)\le V$,
\begin{equation}
\begin{aligned}
\sum_{s=1}^{S}x_{h_s,s}-&\mathbb{E}_{\textrm{Group-EXP3.S}}\left[\sum_{s=1}^{S}x_{\ell_s,s}\right]\le 4\sqrt{VLS\ln (LS)},
\end{aligned}
\end{equation}
where $\ell_s$ is the arm group selected at epoch $s$.
\end{lemma}
\begin{proof}
See Appendix C in the supplementary material.
\end{proof}

The upper bound in Theorem \ref{thm:shiftregretUB} on the expected shifting regret of HLMC against any arbitrary benchmark action sequence with a hardness upper bound $V$ is obtained by combining Lemma 3 and Lemma \ref{lemma:groupregret-S} together.

\begin{proof}[Proof of Theorem \ref{thm:shiftregretUB}]
For an arbitrary benchmark action sequence $a^T$ such that $H(a^T)\le V$, we first construct an alternative benchmark sequence~$b^T$ as follows: suppose the time horizon is partitioned into~$V$ segments:
\begin{align}
[T_1, T_2), [T_2,T_3),...,[T_{V},T_{V+1}),
\end{align}
where $T_1=1,T_{V+1}=T+1$, and $a_t$ is fixed for all $t\in[T_v,T_{v+1})$ (let $j_v$ denote that arm and $h_v$ denote the group it belongs to). Suppose $T_v$ belongs to epoch $s_v$. The alternative benchmark sequence $b^T$ is defined as
\begin{align}
b_t=j_v,~\textrm{if}~s(t)\in[s_v,s_{v+1}),
\end{align}
where $s(t)$ is the epoch to which time $t$ belongs. 

One can check that the action specified by $b^T$ is fixed within each epoch and $H(b^T)\le V$. Moreover, $b^T$ differs from $a^T$ only in the epochs when an action switch happens in $a^T$, i.e., $\{s_v\}_{v=1}^{V}$. Therefore,
\begin{align}
\sum_{t=1}^{T}\left(r_{a_t,t}-r_{b_t,t}\right)\le V\Delta.
\end{align} 
We decompose the expected shifting regret against $a^T$ as:
\begin{equation}
\begin{aligned}
&\mathbb{E}_{\textrm{HLMC}}[R_{a^T}(T)]\\
=&\sum_{t=1}^{T}\left(r_{a_t,t}-r_{b_t,t}\right)+\left(\sum_{t=1}^{T}r_{b_t,t}-\sum_{v=1}^{V}\sum_{s=s_v}^{s_{v+1}-1}|\mathcal{T}_s|x_{h_v,s}\right)\\
&+\Bigg(\sum_{v=1}^{V}\sum_{s=s_v}^{s_{v+1}-1}|\mathcal{T}_s|x_{h_v,s}-\mathbb{E}_{\textrm{HLMC}}\left[\sum_{t=1}^{T}r_{i_t,t}\right]\Bigg)\\
=&R_1(T)+R_2(T)+R_3(T).
\end{aligned}
\end{equation}
Note that $R_1(T)\le V\Delta$. For $R_2(T)$, we have
\begin{equation}
\begin{aligned}
R_2(T)&=\sum_{v=1}^{V}\sum_{s=s_v}^{s_{v+1}-1}\sum_{t\in\mathcal{T}_s}r_{b_t,t}-\sum_{v=1}^{V}\sum_{s=s_v}^{s_{v+1}-1}|\mathcal{T}_s|x_{h_v,s}\\
&\le 2S\sqrt{2\Delta N\ln N},
\end{aligned}
\end{equation}
where the last inequality uses Lemma 3.

For $R_3(T)$, we can show that
\begin{equation}
\begin{aligned}
R_3(T)&= \sum_{v=1}^{V}\sum_{s=s_v}^{s_{v+1}-1}|\mathcal{T}_s|x_{h_v,s}-\mathbb{E}_{\textrm{Group-EXP3.S}}\left[\sum_{s=1}^{S}\Delta x_{\ell_s,s}\right]\\
&\le 4\Delta\sqrt{VLS\ln(LS)},
\end{aligned}
\end{equation}
where the last inequality uses Lemma \ref{lemma:groupregret-S}.

Combining the above inequalities together and choosing $\Delta=\left\lceil\sqrt{\frac{TN\ln N}{VL\ln(TL)}}\right\rceil$, we can derive that
\begin{equation}\label{eq:shift_upper}
\mathbb{E}_{\textrm{HLMC}}[R_{a^T}(T)]\le 6\sqrt{2}T^{\frac{3}{4}}V^{\frac{1}{4}}K^{\frac{1}{4}}(\ln(KT))^{\frac{1}{2}}+\sqrt{TVK\ln K}.
\end{equation}
Notice that if $T\ge VK$, the first term on the RHS of (\ref{eq:shift_upper}) dominates. Since $a^T$ is chosen arbitrarily with a hardness upper bounded $V$, we obtain the conclusion in Theorem 2.
\end{proof}

It should be noted that to achieve the upper bound established in Theorem \ref{thm:shiftregretUB}, the knowledge of $V$ is required in selecting input parameters. When $V$ is unknown, we show in the following theorem that no-regret learning under shifting regret can still be achieved by HLMC in expectation under certain conditions.
\begin{theorem}
By selecting $\Delta=\left\lceil\sqrt{\frac{TN\ln N}{L\ln (TL)}}\right\rceil$ and $\gamma_1=\sqrt{\frac{L\ln(LS)}{S}}$ (the other parameters are identical to those specified in Theorem \ref{thm:shiftregretUB}), the expected shifting regret of HLMC with a hardness constraint $V$ on the benchmark action sequence is upper bounded by:
\begin{equation}
	\mathbb{E}_{\textrm{HLMC}}[R_{\textrm{s}}(T,V)]\le \sqrt{2}(V+5)T^{\frac{3}{4}}K^{\frac{1}{4}}(\ln{(KT)})^{\frac{1}{2}}.
\end{equation}
If $V= o(T^{1/4})$ as $T\to\infty$, no-regret learning is achieved by HLMC in expectation under shifting regret with a hardness constraint $V$, even if $V$ is unknown.
\end{theorem}
\begin{proof}
The proof is similar to that of Theorem \ref{thm:shiftregretUB} and thus, we omit the details.
\end{proof}

\section{Memory Complexity and Regret Performance in General Cases}\label{sec:multi}
As discussed in Sec. \ref{sec:HLMC}, HLMC achieves different operating points on the tradeoff curve between the regret order and memory complexity through selecting different depth $D$ of the adopted hierarchy. To show this, we provide performance analysis of HLMC in the general case with $D\ge 2$.  For simplicity, we present detailed analysis for the case with~$D=3$. All claims and results can be easily generalized to cases with more than three levels.

We first introduce some notations and specify some parameters used in the algorithm as well as the analysis. The three levels in the hierarchy are referred to as the group, subgroup, and arm levels, respectively. In the first level, the arm set~$\mathcal{A}$ is evenly partitioned into $N_1=\lceil K^{1/3}\rceil$ groups $\{\mathcal{A}_{\ell}\}_{\ell=1}^{N_1}$. Within each group~$\mathcal{A}_{\ell}$, arms are further evenly partitioned into $N_2=\lceil K^{1/3}\rceil$ subgroups $\{\mathcal{B}_{h}^{\ell}\}_{h=1}^{N_2}$ in the second level. In the last level, each subgroup~$\mathcal{B}_{h}^{\ell}$ consists of~$N_3=\lceil\frac{K}{N_1N_2}\rceil$ arms (the size of the last subgroup within each group may be smaller than $N_3$). We assume without losing generality that the size of every group (subgroup) is identical. Similarly, the time horizon~$\mathcal{T}$ is evenly partitioned into $S_1$ epochs~$\{\mathcal{T}_s\}_{s=1}^{S_1}$ and every epoch $\mathcal{T}_s$ is evenly partitioned into $S_2$ subepochs~$\{\mathcal{I}_{\tau}^{s}\}_{\tau=1}^{S_1}$. We assume that every sub-epoch consists of $S_3$ time steps ($S_1,S_2,S_3$ will be specified later). It is clear that $T=S_1S_2S_3$.

The HLMC framework consists of three selection strategies at the group, subgroup, and arm levels. At the beginning of every epoch $\mathcal{T}_s$, the group-level strategy selects a group~$\mathcal{A}_{\ell_s}$. The statistics of all sub-groups within $\mathcal{A}_{\ell}$ are stored in the memory until the end of $\mathcal{T}_s$. During $\mathcal{T}_s$, the subgroup-level strategy selects a subgroup $\mathcal{B}_{h_\tau}^{\ell_s}$ at the beginning of every subepoch $\mathcal{I}_{\tau}^{s}$ and the statistics of arms within $\mathcal{B}_{h_\tau}^{\ell_s}$ are stored in the memory until the end of $\mathcal{I}_{\tau}^{s}$. The arm-level strategy is conducted on the selected subgroup to play arms at every time step during the corresponding subepoch.

It is clear that the size of the memory space required by HLMC with a three-level hierarchy is $N_1+N_2+N_3$. Therefore, the memory complexity of HLMC is in the order of $\Theta(K^{1/3})$. More generally, if we adopt a $D$-level hierarchy where each level $d$ $(d=1,2,...,D)$ consists of $N_d=\lceil K^{1/D}\rceil$ level-$d$ groups, the memory complexity of HLMC is of order~$\Theta(DK^{1/D})$. It should be noted that a level-$d$ group should contain at least $2$ level-$(d+1)$ groups. As a result, the depth~$D$ is upper bounded by $\lceil\log_2 K\rceil$ and the minimum memory complexity of the HLMC framework is of order $\Theta(\log_2K)$.

We show that HLMC with a three-level hierarchy achieves no-regret learning in expectation under the notion of weak regret, if we adopt EXP3 at all three levels. Using a similar approach with that in analyzing the regret performance in the two-level case, we prove an upper bound on the expected weak regret of HLMC in the following theorem.

\begin{theorem}\label{thm:weakregretML}
For any $T$ and $K$, by choosing $S_i=\left\lceil\frac{T^{1/3}(N_i\ln N_i)^{2/3}}{(\prod_{j\neq i}N_j\ln N_j)^{1/3}}\right\rceil$ and applying EXP3 with parameter $\gamma_i=\sqrt{\frac{N_i\ln N_i}{2S_i}}$ at every level $i=1,2,3$, the expected weak regret of HLMC with a three-level hierarchy against every assignment of the reward sequence is upper bounded~by
\begin{equation}
\mathbb{E}_{\textrm{HLMC}}[R_{\textrm{w}}(T)]\le 12T^{5/6}K^{1/6}(\ln K)^{1/2}.
\end{equation}
\end{theorem}

\begin{proof}
See Appendix D in the supplementary material.
\end{proof}

For general HLMC with a $D$-level hierarchy ($2\le D\le \lceil\log_2 K\rceil$), the following corollary on the expected weak regret can be directly derived.
\begin{corollary}\label{cor:weakregretML}
If EXP3 is applied to all $D$ levels of the general HLMC framework, the expected weak regret is of order
\begin{equation}
O(DT^{1-\frac{1}{2D}}K^{\frac{1}{2D}})
\end{equation}
up to a logarithmic factor, as $T\to\infty$.
\end{corollary}
\begin{proof}
The proof is similar to the ones of Theorem \ref{thm:weakregretUB} and \ref{thm:weakregretML} and thus, we omit the details.
\end{proof}

Corollary \ref{cor:weakregretML} indicates that the tradeoff between the regret order and memory complexity of HLMC depends on the depth~$D$ of the adopted hierarchy: a deeper hierarchy incurs a higher regret order with a smaller memory complexity. We further establish a memory-dependent regret upper bound of HLMC by adaptively selecting $D$ based on the size $M$ of the available memory space. In particular, we define the minimum depth $D^*(M)$ of a legitimate hierarchy when $M$ words of memory are available:
\begin{equation}
D^* (M)=\min\{D\in\mathbb{N}^+: D\lceil K^{1/D}\rceil\le M\}.
\end{equation}
Thus, the minimum regret achieved by HLMC with $M$ words of memory is of order 
\begin{equation}
O\left(D^*(M)T^{1-\frac{1}{2D^*(M)}}K^{\frac{1}{2D^*(M)}}\right).
\end{equation}

In one extreme case when $M=\Theta(\log_2 K)$, the size of the available memory space matches the minimum complexity of the deepest hierarchy where $D^*(M)=\lceil\log_2 K\rceil$. In this case, the regret order achieved by HLMC is still sublinear in~$T$. In the other extreme case when $M\ge K$ (i.e., the memory-unconstrained case), it is clear that $D^*(M)=1$ and HLMC with a single-level hierarchy reduces to an existing learning routine for memory-unconstrained adversarial bandits.

One may notice that the memory-dependent regret order of HLMC does not improve when $M$ increases but $D^*(M)$ is unchanged, since the dependency of the regret order with respect to the available memory is quantified. 
However, in practice, a larger memory space may help in achieving a smaller regret if arms are adaptively partitioned according to~$M$, even if $D$ is fixed. We take the two-level case as an example: given $M$ words of memory, we let $N=\lceil\frac{M-\sqrt{M^2-4K}}{2}\rceil$ and $L=\lceil\frac{K}{N}\rceil$. As long as $M\ge 2\sqrt{K}$, the arm partition is legitimate and one can verify that  $N+L\le M$. It is not difficult to check that the theoretical regret orders established in Sec.~\ref{sec:analysis} still hold under the adaptive arm partition. We further show in Sec. \ref{subsec:regretvsmemory} through numerical examples that under certain conditions, the regret performance of HLMC using adaptive arm partitions in the two-level hierarchy improves as $M$ increases.

\section{Numerical Examples}\label{sec:numerical}
In this section, we illustrate the regret performance of the proposed HLMC learning structure numerically through simulations. All the experiments are run 10 times using a Monte Carlo method on Python 3.7.

\subsection{Weak Regret Minimization}\label{subsec:numericalweak}
We conduct two experiments to compare the regret performance of HLMC with baseline ones under the notion of weak regret. Given that this is the first work on memory-constrained adversarial bandits, we consider two baselines: UCB-M (proposed in \cite{chaudhuri2019regret} for memory constrained stochastic bandits) and EXP3 (for classic adversarial bandits without memory constraints).

We first notice that the only randomness of UCB-M comes from the random shuffle of arm indices before playing arms, which provides no improvement on the performance in the stochastic setting. Without the random shuffle step, UCB-M is purely deterministic and thus, we can easily construct a reward sequence such that UCB-M incurs a regret linear in $T$. Specifically, in the first experiment, we consider the following setup: let $K=100$, $M=20$, and $T=10^7$. In accordance with the UCB-M policy, we partition the time horizon into phases with exponentially growing lengths $2^ih_0b_0$ ($i=0,2,...$). Each phase is further partitioned evenly into~$h_0$ sub-phases with length $2^ib_0$. We select $h_0=\lceil\frac{K-1}{M-1}\rceil$ and $b_0=M(M+2)$. For each phase, we assign arm rewards as follows: during each subphase $u=0,2,...,h_0-1$, we let arm $(M(u+1)\textrm{~mod~}K)$ offer reward $1$ and the other arms offer reward $0$. Since UCB-M selects arm groups with size $M$ in a round-robin fashion, it is clear that arms selected by UCB-M offers $0$ reward at almost all time steps. The weak regret of UCB-M is clearly linear in $T$. For HLMC, we adopt a two-level hierarchy and apply EXP3 to both group and arm levels. The simulation results on the expected weak regret are presented in Fig. \ref{fig:weakregret1}.

\begin{figure}[t]
\hspace{.2cm}
\centerline{\includegraphics[width=1.1\columnwidth]{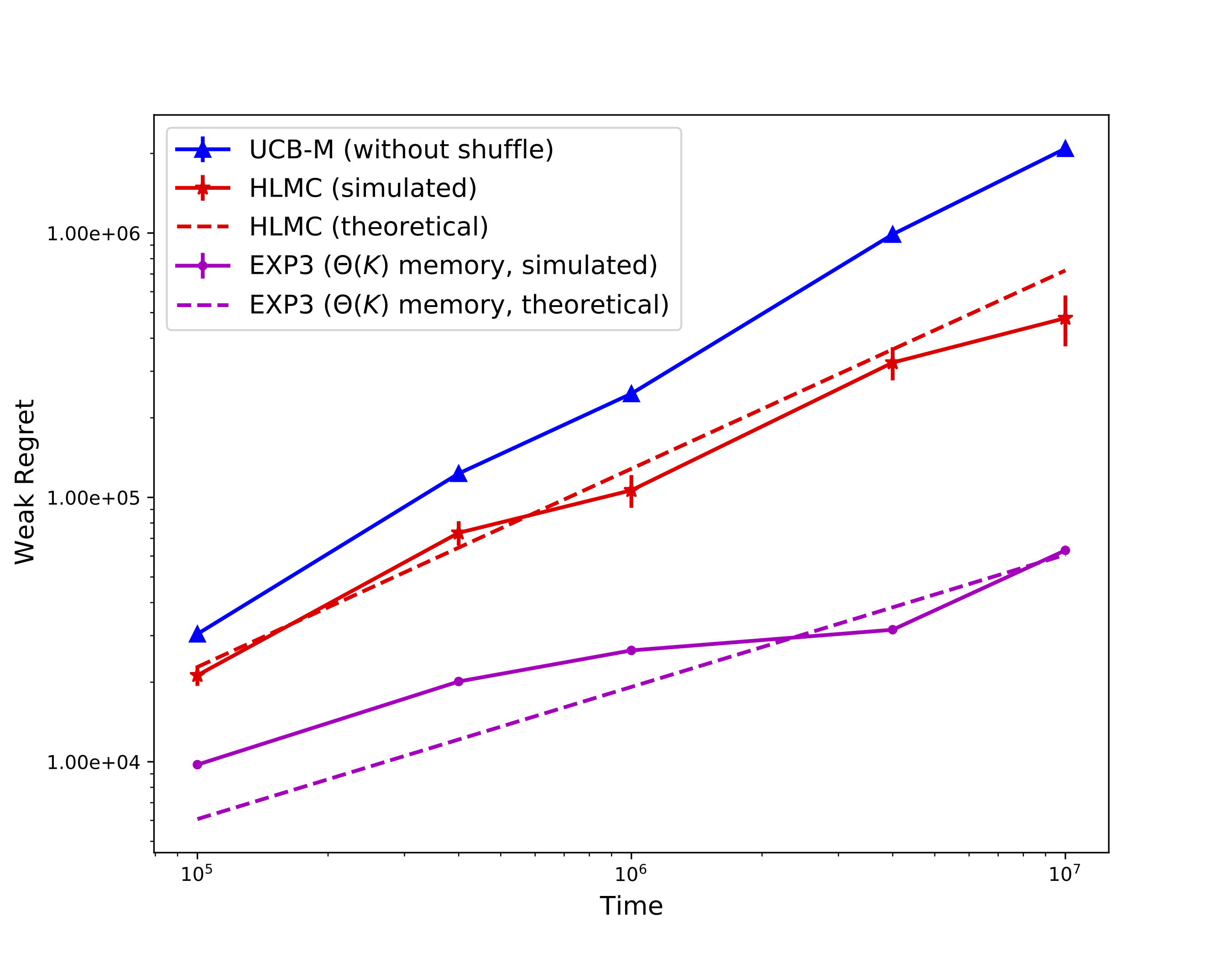}}
\caption{Comparison of the weak regret of UCB-M (without random shuffle of arm indices), HLMC, and EXP3:~$K=100$, $M=20$, and $T=10^7$. The time horizon is partitioned into phases with exponentially growing lengths~$2^ih_0b_0$ ($i=0,2,...$). Each phase is partitioned evenly into~$h_0$ sub-phases with length $2^ib_0$. During each subphase $u$, arm $(M(u+1)\textrm{~mod~}K)$ offers reward $1$ and the other arms offer reward $0$.}
\label{fig:weakregret1}
\end{figure}

From Fig. \ref{fig:weakregret1}, we can observe that HLMC outperforms the UCB-M policy under the constructed adversarial environment. The error bar indicates that the proposed learning policy is robust with low variance. Note that although the EXP3 algorithm achieves the best performance, it requires $\Theta(K)$ memory size, which is infeasible in the memory-constrained setting. We also plot the theoretical upper bounds on the regret of HLMC and EXP3 (i.e.,~$2T'^{\frac{3}{4}}K^{\frac{1}{4}}(\ln K)^{\frac{1}{2}}$ and $2\sqrt{T'K\ln K}$ where $T'=T/5$ due to the fact that the cumulative reward of the best arm is $T/5$ instead of $T$ in this experiment\footnote{The choice of the constant in front of $T, K$ does not change the regret order. To demonstrate that the theoretical regret bound and the simulated results have the same order, we set the constant equal to $2$.}), which verify that the expected weak regret of HLMC has the same order with the theoretical upper bounded established in Theorem \ref{thm:weakregretUB}.

We further use another example to show that even with the random shuffle step, UCB-M still fails to avoid a linear regret in $T$ against adversaries. We consider the same experiment setup with a different reward assignment. Specifically, the phase and subphase partitions are the same with those in the first experiment. During each subphase $u=0,2,...,h_0-1$, we let arm $1$ offer $(u\textrm{~mod~}2)$ reward and the other arms offer $\epsilon=1\times 10^{-4}$ rewards. It is not difficult to check that after every time arm $1$ is selected by UCB-M and offers reward $1$, it will offer $0$ reward in the next subphase and will be excluded from memory. Therefore, significant regret is incurred in the subphase after next, when arm $1$ offers $1$ reward again. Over the entire time horizon, UCB-M suffers a linear regret order in~$T$. Moreover, we added another baseline: EXP3-M by changing the UCB subroutine in UCB-M to the EXP3 subroutine. The simulation results are presented in Fig. \ref{fig:weakregret2}, which again verify the advantage of HLMC against UCB-M and EXP3-M. It should be noted that even with the random shuffle step or a subroutine developed for classic adversarial bandits during every epoch, the UCB-M and EXP3-M algorithms still suffer significant regret due to the fact that the algorithmic structure of the two algorithms fails to balance between what to remember and what to forget in the adversarial setting. Besides, the random shuffle step in UCB-M and EXP3-M introduces high variance with little improvement on the expected weak regret. The comparison between the theoretical upper bounds and the simulated results also verifies the correctness of our analysis in Theorem \ref{thm:weakregretUB}.

\begin{figure}[t]
\begin{center}
\centerline{\includegraphics[width=1.1\columnwidth]{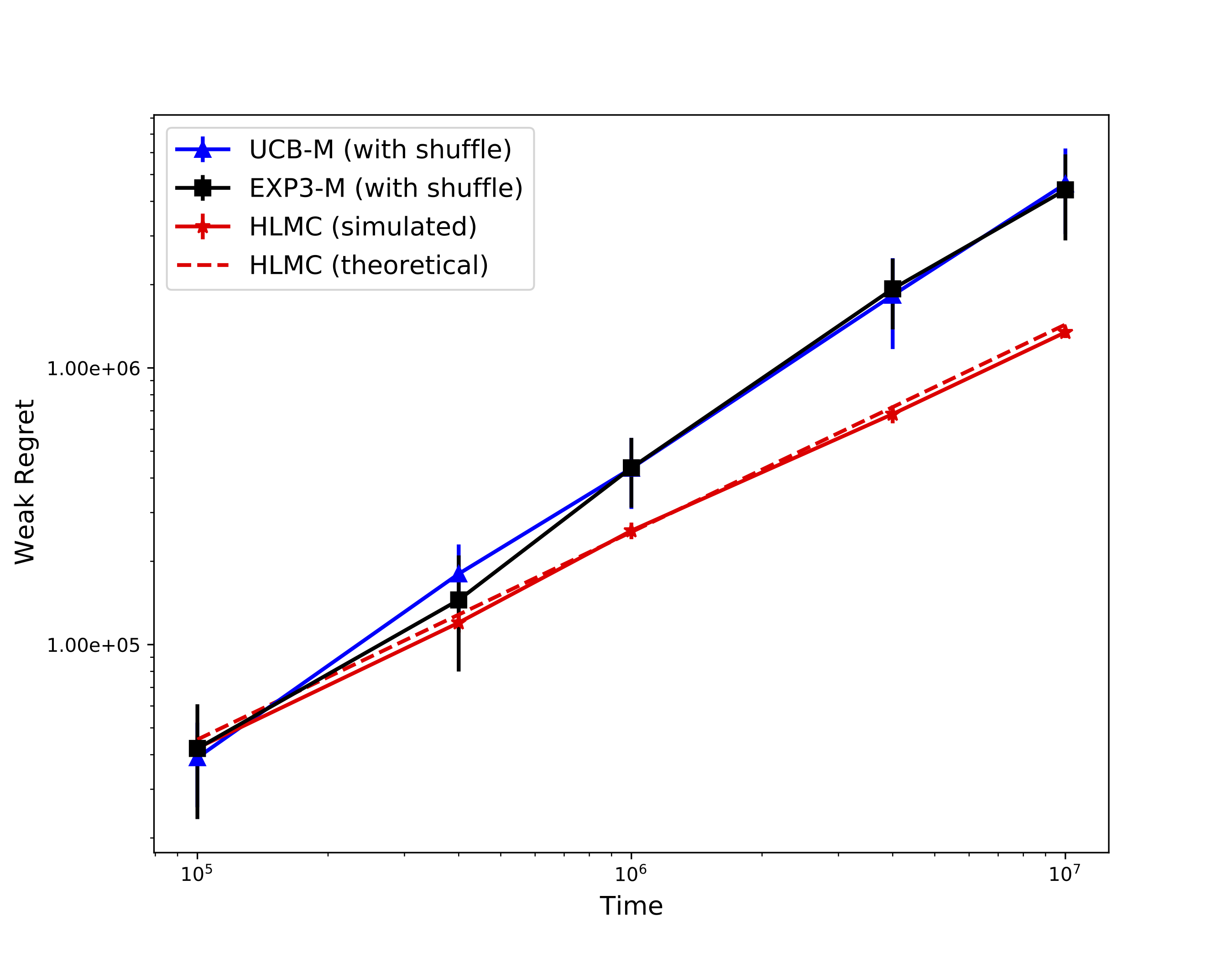}}
\caption{Comparison of the weak regret of UCB-M (with random shuffle of arm indices), EXP3-M (with random shuffle of arm indices) and HLMC: the time partition is the same with that in Fig. \ref{fig:weakregret1}. During each subphase~$u$, arm~1 offers $(u\textrm{~mod~}2)$ reward and the other arms offer $\epsilon=10^{-4}$ reward.}
\label{fig:weakregret2}
\end{center}
\end{figure}

\subsection{Shifting Regret Minimization}
We further conduct an experiment to show the regret performance of HLMC with a two-level hierarchy under the notion of shifting regret. As discussed in Sec. \ref{subsec:shiftregret}, by adopting EXP3.S at the group level, HLMC achieves a sublinear scaling of shifting regret in $T$. In this experiment, we compare the performance of HLMC adopting EXP3.S at the group level and EXP3 at the arm level (referred to as HLMC.S in this subsection), HLMC adopting EXP3 at both group and arm levels (referred to as HLMC in this subsection), EXP3, and EXP3.S. The experiment is set up as follows: let $K=16, M=8$, and $T=10^6$. The time horizon is partitioned evenly into $V=10$ phases. In phase $v=0,1,...,V-1$, we let arm $i_v=(vN\textrm{~mod~} K)$ offer reward $1$ and the other arms offer reward $0$ ($N$ is the group size defined in the HLMC framework, which equals $4$ in this experiment). It is clear that the best benchmark policy in the shifting regret definition with hardness $V$ is to play the best arm $i_v$ within every phase $v$. The simulation results are presented in Fig. \ref{fig:shiftregret}.

\begin{figure}[t]
\begin{center}
\centerline{\includegraphics[width=1.1\columnwidth]{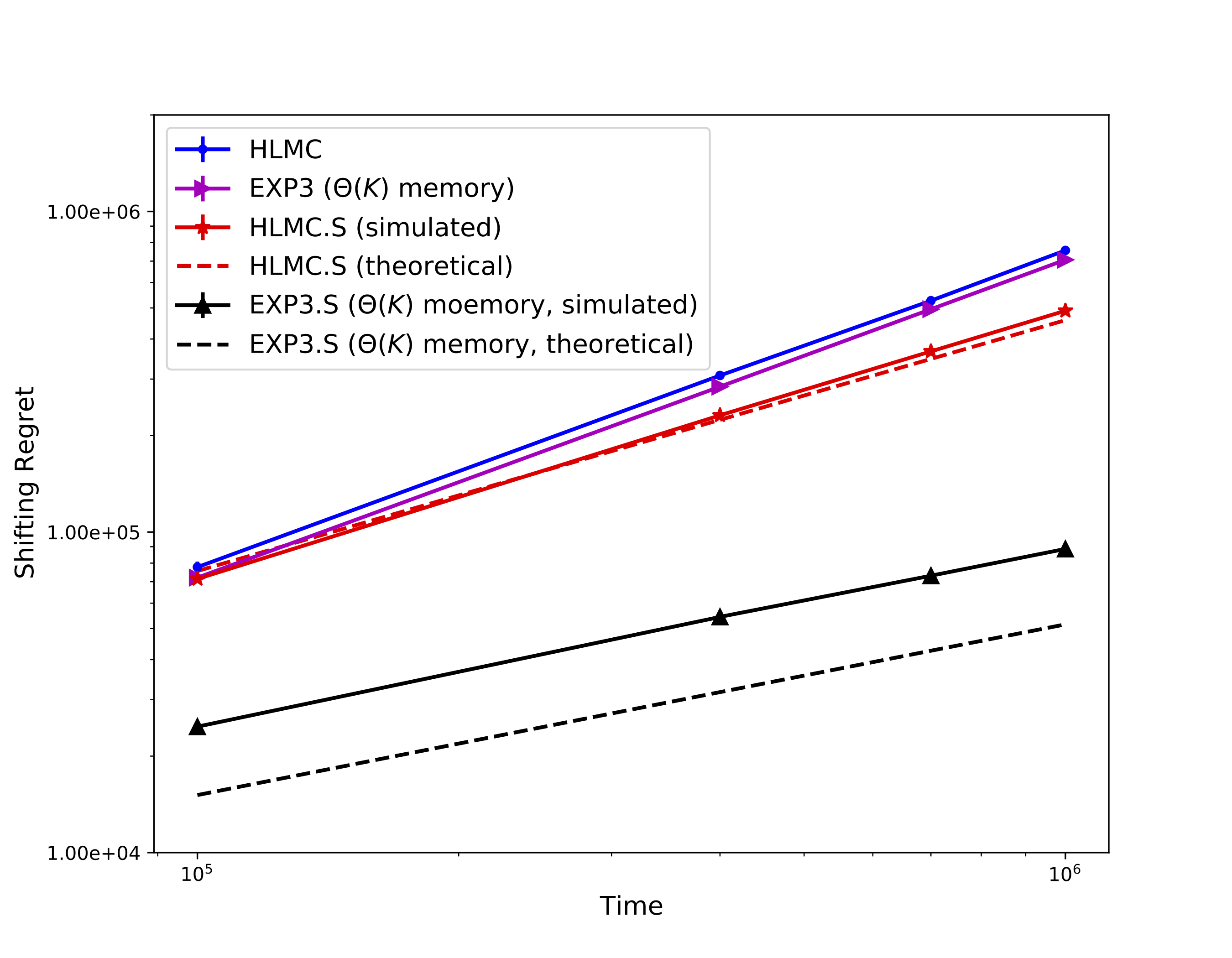}}
\caption{Comparison of the shifting regret of HLMC.S, HLMC, EXP3, and EXP3.S: $K=16, M=8$, and $T=10^6$. The time horizon is partitioned evenly into $V=10$ phases. In phase $v=0,1,...,V-1$, arm $i_v=(vN\textrm{~mod~} K)$ offer reward $1$ and the other arms offer reward $0$.}
\label{fig:shiftregret}
\end{center}
\end{figure}

It can be observed from Fig. \ref{fig:shiftregret} that HLMC.S designed for shifting regret minimization outperforms HLMC and EXP3 for weak regret minimization. Adopting EXP3.S at the group level of the HLMC framework improves the regret performance under the notion of shifting regret. Moreover, the error bar verifies the robustness of the proposed policies. It should be noted that although EXP3.S outperforms HLMC and HLMC.S, it requires $\Theta(K)$ memory space, which is inapplicable in the memory-constrained setting.

\subsection{Impact of Available Memory on Regret Performance}\label{subsec:regretvsmemory}
In this subsection, we show the impact of the size of available memory space on the regret performance of HLMC. We use the same experiment setup with that in the first experiment in Sec. \ref{subsec:numericalweak}. We compare the weak regret of HLMC with $M=14,20,50,80$. Specifically, when $M=14$, the HLMC framework requires a three-level hierarchy with $N_1=5, N_2=5$, and $N_3=4$. When $M=20,50,80$, HLMC adopts two-level hierarchies with $N=\lceil\frac{M-\sqrt{M^2-4K}}{2}\rceil$ and $L=\lceil K/N\rceil$. 

The results in Fig. \ref{fig:weakregret_memory} show that the regret performance of HLMC improves as the size of the memory space increases. In particular, adopting a hierarchy with fewer levels improves the regret order as indicated in Corollary \ref{cor:weakregretML}. Even with the same number of levels, a smaller regret can be achieved with a larger memory space. Intuitively, as $M$ increases, the epoch length $\Delta$ decrease. Since the reward sequence assigned in the experiment is stable within a short period but varies vastly in the long run (it has been argued in \cite{zimmert2019beating} that such a reward assignment is justified in various real-world applications), the arm-level regret is dominated by the group-level regret and the latter decreases with the epoch length. We also plot the theoretical upper bounds on the regret of HLMC with different levels of hierarchies. The comparison between the theoretical the simulated results verifies our analysis.

\begin{figure}[t!]
\begin{center}
\centerline{\includegraphics[width=1.1\columnwidth]{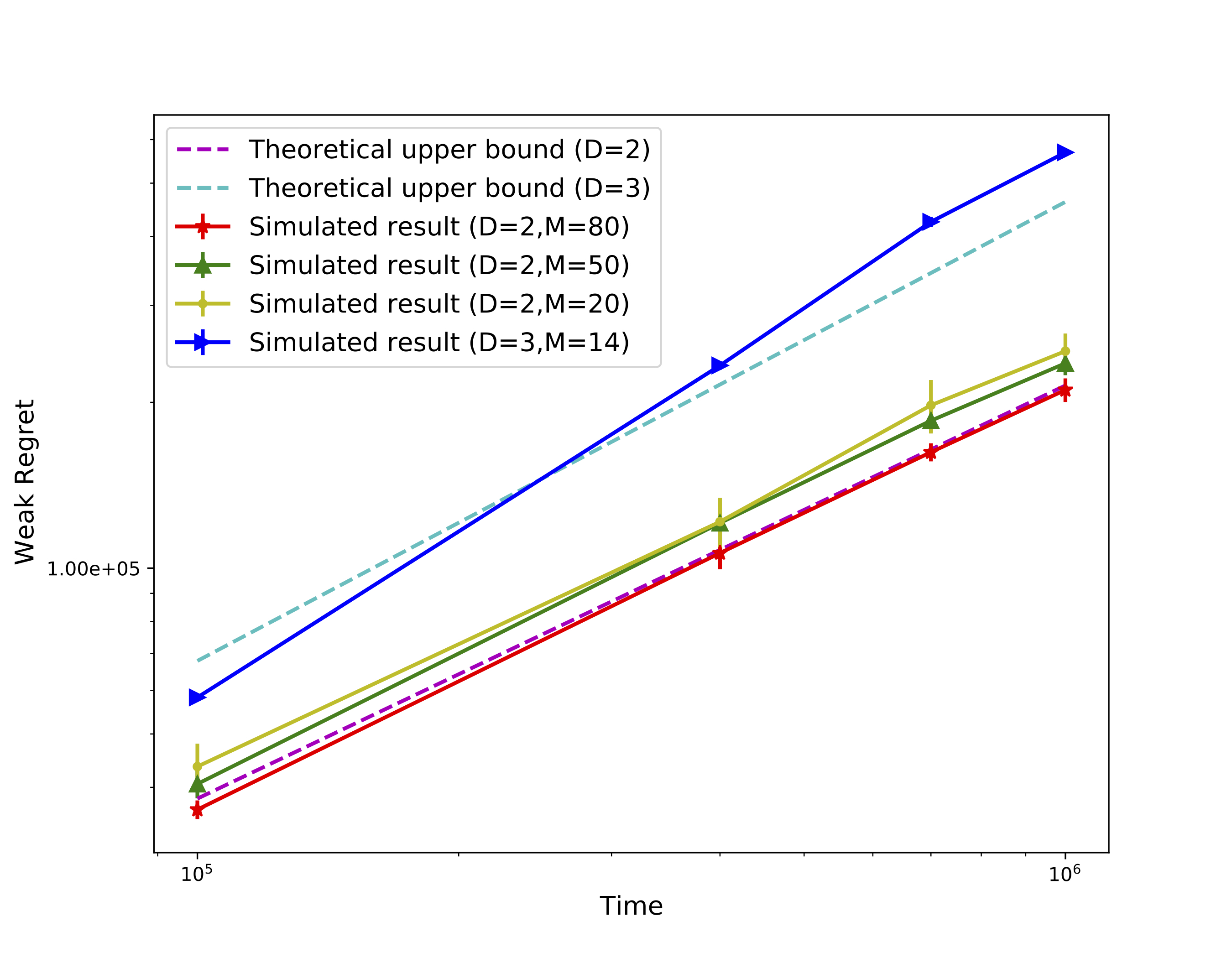}}
\caption{Comparison of the weak regret of HLMC with $M=14,20,50,80$ memory space: the experiment setup is the same with that in Fig. \ref{fig:weakregret1}. When $M=14$, the HLMC framework adopts a three-level hierarchy with $N_1=5, N_2=5$, and $N_3=4$. When $M=20,50,80$, HLMC adopts two-level hierarchies with $N=\lceil\frac{M-\sqrt{M^2-4K}}{2}\rceil$ and $L=\lceil K/N\rceil$.}
\label{fig:weakregret_memory}
\end{center}
\end{figure}

\subsection{Distributed Dynamic Spectrum Access in the Presence of Jamming}
In this subsection, we consider the application of distributed dynamic spectrum access in the presence of jamming in multi-agent wireless communication systems. There are $1000$ distributed agents competing for $K=20$ channels (arms) and an attacker that is jamming the channels. The transmission rate of a channel is modeled as the reward of the corresponding arm. The quality of a channel depends on whether it is jammed by the attacker and how many distributed agents are accessing the channel simultaneously. Specifically, we assume that if a channel is not jammed, it offers reward $10$ and all accessing agents evenly share the reward, i.e., every agent receives $10/n_t$ reward where $n_t$ is the number of agents selecting the unjammed channel at time $t$. For the jammed channels, an agent can only receive $1$ reward if there is no collision (if there are more than two agents selecting the same arm, no agent can receive reward from this arm). In this experiment, we consider an attacker that jams all but one channel at every time step $t$ and the unjammed channel changes at the beginning of every phase and circulates among the $K$ channels.

From the perspective of every agent, the problem can be modeled as a memory-constrained adversarial bandit problem studied in this paper. Due to limited memory on distributed wireless devices, every agent can store at most $M=10$ statistics of arm rewards. We compare the per-agent average reward where each agent adopts HLMC with that adopting UCB-M (with and without random shuffle). The simulation result is shown in Fig. \ref{fig:channel_selection}, which again demonstrates the advantage of HLMC against UCB-M in the adversarial setting. 

\begin{figure}[t!]
\begin{center}
\centerline{\includegraphics[width=1.1\columnwidth]{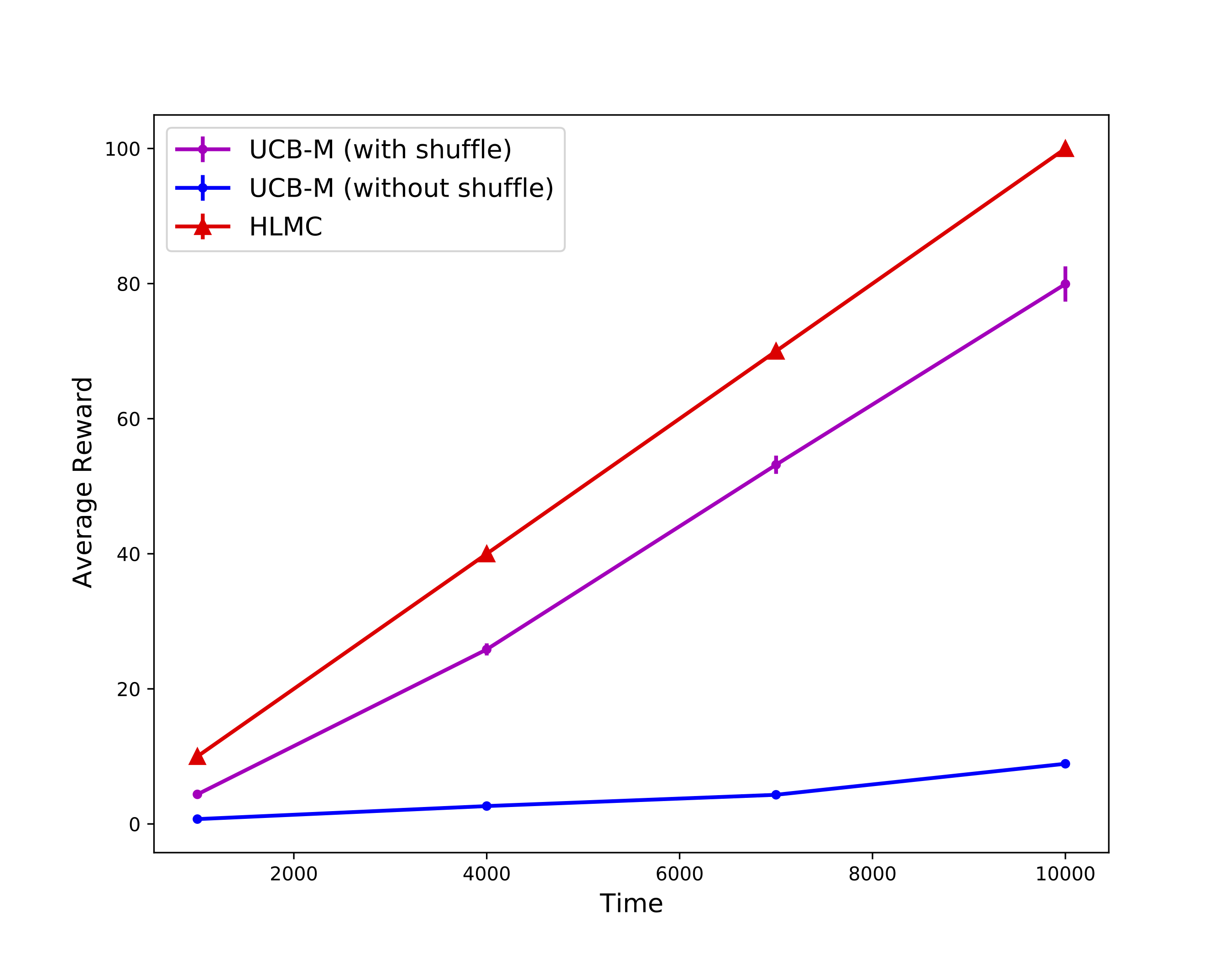}}
\caption{Comparison of the average reward per agent by adopting HLMC and UCB-M (with and without random shuffle of arm indices) in dynamic spectrum access in the presence of jamming: $1000$ distributed agents compete for $K=20$ channels. An attacker adversarially jams all but one channel at every time step and the unjammed channel circulates among the $K$ channels. For the unjammed channel, all accessing agents evenly share $10$ reward. For the jammed channel, an agent can only receive 1 reward if there is no collision.}
\label{fig:channel_selection}
\vspace{-.5cm}
\end{center}
\end{figure}

\section{Conclusions and Discussions}
In this paper, we studied the problem of adversarial multi-armed bandits with memory constraints. We proposed a general hierarchical learning framework: HLMC that adopts a multi-level hierarchy to partition the arms into groups and the time horizon into epochs. The HLMC framework decouples the tradeoff between what to remember and what to forget induced by memory constraints from the one between exploration and exploitation due to bandit feedback. We showed in the two-level case that, by employing different existing learning routines developed for memory-unconstrained bandits at both levels of the hierarchy, HLMC achieves no-regret learning under various regret notions with a memory complexity sublinear in the number of arms. We further showed that through designing the depth of the adopted hierarchy, HLMC achieves different operating points at the tradeoff curve between the regret order and memory complexity. We conducted numerical experiments to verify the advantages of HLMC against existing baselines.

Several questions remain open in this problem. It is unclear whether $\Theta(\log K)$ is the minimum memory complexity required for achieving no-regret learning in the adversarial setting. Moreover, the current hierarchical partition of arm groups is pre-determined. It worth studying whether a dynamic (potentially stochastic) grouping strategy that depends on past observations can improve the regret performance. More importantly, whether the sequence of operating points offered by the proposed algorithm traces the Pareto front
of this fundamental tradeoff between regret performance and memory complexity is an interesting open question that requires a separate full investigation. Another potential research direction is to find the best of both worlds, that is, a learning policy achieving the optimal regret orders in both stochastic and adversarial settings with memory constraints.


%

\ifCLASSOPTIONcaptionsoff
  \newpage
\fi



%
\bibliographystyle{IEEEtran}
\bibliography{Reference}

\onecolumn
\appendices
\section{Existing Memory-Unconstrained Learning Routines}\label{Existing Algorithm}

The EXP3 algorithm was first proposed in \cite{auer1995gambling} to minimize the expected weak regret. Following this algorithm, a player randomly selects an action $i_t$ according to a distribution~$(p_{i,t})_{i\in\mathcal{A}}$ at every time $t$. The probability $p_{i,t}$ is the sum of two components. The first one is proportional to a weight $w_{i,t}$ exponential in the estimated cumulative reward from arm $i$ up to time $t$, i.e., $w_{i,t}=\prod_{\tau=1}^{t}\exp(\gamma\hat{r}_{i,\tau}/K)$, where~$\gamma>0$ is the learning rate and $\hat{r}_{i,\tau}$ is an unbiased estimate of $r_{i,t}$ with respect to the random arm selection. The second component is a random exploration term $\gamma/K$ ensuring sufficient exploration of every arm. The details of EXP3 are summarized in Algorithm \ref{alg:EXP3}. It has been shown that EXP3 achieves a sublinear regret order in $T$ under the notion of expected weak regret.

\begin{algorithm}[h!]
   \caption{EXP3 \cite{auer1995gambling,auer2002nonstochastic}}
   \label{alg:EXP3}
   \begin{algorithmic}
   \STATE {\bfseries Input:} $\mathcal{A}$ the arm set and $\gamma\in(0,1)$.

      \STATE {\bfseries Initialization:} $w_{i,1}=1,\forall i\in\mathcal{A}$.
       \FOR{$t=1,2,...,T$}
   	        \STATE Let $$p_i(t)=(1-\gamma)\frac{w_{i,t}}{\sum_{j\in\mathcal{A}}w_{j,t}}+\frac{\gamma}{K}, ~~~\forall i\in\mathcal{A}.$$
   	        \STATE Draw arm $i_t$ according to the probabilities $(p_{i,t})_{i\in\mathcal{A}}$.
   	        \STATE Receive reward $r_{i_t,t}$.
   	        \STATE Let $$\hat{r}_{i,t}=\frac{r_{i,t}}{p_{i,t}}\mathbb{I}(i_t=i).$$
   	        \STATE Update
   	        $$
   	        w_{i,t+1}=w_{i,t}\exp\left(\frac{\gamma\hat{r}_{i,t}}{K}\right),~~~\forall i\in\mathcal{A}.
   	        $$
       \ENDFOR

\end{algorithmic}
\end{algorithm}

\begin{lemma}[Corollary 3.2 in \cite{auer2002nonstochastic}]\label{lemma:armregret}
By choosing $\gamma=\sqrt{\frac{K\ln K}{2T}}$, the expected weak regret of EXP3 over a time horizon of length~$T$ is upper bounded by
\begin{align}
\max_{i\in\mathcal{A}}\sum_{t=1}^{T}r_{i,t}-\mathbb{E}_{\textrm{EXP3}}\left[\sum_{t=1}^Tr_{i_t,t}\right]\le 2\sqrt{2T K\ln K}
\end{align}
for every assignment of the reward sequence, where $i_t$ is the arm selected by EXP3 at time $t$.
\end{lemma}

To achieve no-regret learning under weak regret with high probability, the EXP3.P algorithm modifies EXP3 by adding an upper confidence term to the unbiased estimate $\hat{r}_{i,t}$ in updating arm weights. This modification guarantees that the true reward is upper bounded by the new estimate with high probability. The details of EXP3.P is summarized in Algorithm~\ref{alg:EXP3P} and its regret performance is shown in Lemma 4.

\begin{algorithm}[h!]
   \caption{EXP3.P \cite{bubeck2012regret}}
   \label{alg:EXP3P}
   \begin{algorithmic}
   \STATE {\bfseries Input:} $\mathcal{A}$ the arm set, $\eta>0$, and $\gamma,\beta\in(0,1)$.
      \STATE {\bfseries Initialization:} $w_{i,1}=1,\forall i\in\mathcal{A}$.
       \FOR{$t=1,2,...,T$}
   	        \STATE Let $$p_i(t)=(1-\gamma)\frac{w_{i,t}}{\sum_{j\in\mathcal{A}}w_{j,t}}+\frac{\gamma}{K}, ~~~\forall i\in\mathcal{A}.$$
   	        \STATE Draw arm $i_t$ according to the probabilities $(p_{i,t})_{i\in\mathcal{A}}$.
   	        \STATE Receive reward $r_{i_t,t}$.
   	        \STATE Let 
   	        $$\tilde{r}_{i,t}=\frac{r_{i,t}\mathbb{I}(i_t=i)+\beta}{p_{i,t}}.$$ 
   	        \STATE Update
   	        $$
   	        w_{i,t+1}=w_{i,t}\exp\left(\eta\tilde{r}_{i,t}\right),~~~\forall i\in\mathcal{A}.
   	        $$
       \ENDFOR
\end{algorithmic}
\end{algorithm}

\begin{lemma}[Theorem 3.2  in \cite{bubeck2012regret}]\label{lemma:highprobregret}
For every $\delta_0\in(0,1)$, by choosing $\beta=\sqrt{\frac{\ln(K/\delta_0)}{KT}},\eta=0.95\sqrt{\frac{\ln K}{KT}}, \gamma=1.05\sqrt{\frac{K\ln K}{T}}$, the EXP3.P algorithm guarantees that, for every assignment of the reward sequence, 
\begin{equation}
\max_{i\in\mathcal{A}}\sum_{t=1}^{T}r_{i,t}-\sum_{t=1}^Tr_{i_t,t}\le 5.15\sqrt{KT\ln(K/\delta_0)}
\end{equation}
with probability at least $1-\delta_0$, where $i_t$ is the arm selected by EXP3.P at time $t$.
\end{lemma}

To minimize shifting regret, the EXP3.S algorithm differs from EXP3 in that a fixed share of the weights from all arms is added to the update process of every arm, i.e., $w_{i,t+1}=w_{i,t}\exp(\gamma\hat{r}_{i,t}/K)+\alpha W_t$ where $W_t=\sum_{j\in\mathcal{A}}w_{j,t}$. It has been shown in \cite{herbster1998tracking} that by sharing a portion of arm weights, the effect of past rewards on future arm selection diminishes. As a result, arm selection relies more on recent rewards to adapt to the time-varying benchmark action sequence. The detailed EXP3.S algorithm is summarized in Algorithm~\ref{alg:EXP3S} and its regret performance is shown in Lemma 5.

\begin{algorithm}[h!]
   \caption{EXP3.S \cite{auer2002nonstochastic}}
   \label{alg:EXP3S}
   \begin{algorithmic}
   \STATE {\bfseries Input:} $\mathcal{A}$ the arm set, $\gamma\in(0,1)$, and $\alpha>0$.
      \STATE {\bfseries Initialization:} $w_{i,1}=1,\forall i\in\mathcal{A}$.
       \FOR{$t=1,2,...,T$}
   	        \STATE Let $$p_i(t)=(1-\gamma)\frac{w_{i,t}}{\sum_{j\in\mathcal{A}}w_{j,t}}+\frac{\gamma}{K}, ~~~\forall i\in\mathcal{A}.$$
   	        \STATE Draw arm $i_t$ according to the probabilities $(p_{i,t})_{i\in\mathcal{A}}$.
   	        \STATE Receive reward $r_{i_t,t}$.
   	        \STATE Let $$\hat{r}_{i,t}=\frac{r_{i,t}}{p_{i,t}}\mathbb{I}(i_t=i).$$
   	        \STATE Update
   	        $$
   	        w_{i,t+1}=w_{i,t}\exp\left(\frac{\gamma\hat{r}_{i,t}}{K}\right)+\frac{e\alpha}{K}\sum_{i\in\mathcal{A}}w_{i,t}.
   	        $$
       \ENDFOR

\end{algorithmic}
\end{algorithm}

\begin{lemma}[Corollary 8.3 in \cite{auer2002nonstochastic}]\label{lemma:shiftregret}
By choosing $\gamma=\sqrt{\frac{KV\ln (KT)}{T}}$ and $\alpha=1/T$, the EXP3.S algorithm guarantees that, for every assignment of the reward sequence $((r_{1,t},...,r_{K,t}))_{t=1}^{T}$ and every benchmark arm sequence $a^T$ with $H(a^T)\le V$,
\begin{equation}
\begin{aligned}
\sum_{t=1}^{T}r_{a_t,t}-&\mathbb{E}_{\textrm{EXP3.S}}\left[\sum_{t=1}^{T}r_{i_t,t}\right]\le 4\sqrt{VKT\ln (KT)},
\end{aligned}
\end{equation}
where $i_t$ is the arm selected by EXP3.S at time $t$.
\end{lemma}
\section{Proof of Lemma \ref{lemma:groupregret}}\label{apx:groupregret}
Let $q_{\ell,s}=\prod_{\sigma=1}^{s}\exp(\gamma_1\hat{y}_{\ell,\sigma})$ denote the weight of group $\ell$ at epoch $s$ where
\begin{equation}
\begin{aligned}
\hat{y}_{\ell,s}&=\frac{y_{\ell,s}}{q_{\ell,s}}\mathbb{I}(\ell_s=\ell)\le \frac{\gamma_1}{L},\\
q_{\ell,s}&=(1-\gamma_1)\frac{g_{\ell, s}}{\sum_{\ell=1}^{L}g_{\ell,s}}+\frac{\gamma_1}{L}\ge \frac{\gamma_1}{L}.
\end{aligned}
\end{equation}
Let $G_{s}=\sum_{\ell=1}^{L}g_{\ell,s}$. We have
\begin{equation}\label{eq:Gt}
\begin{aligned}
\frac{G_{s+1}}{G_{s}}&=\sum_{\ell=1}^{L}\frac{g_{\ell,s}e^{\frac{\gamma_1\hat{y}_{\ell,s}}{L}}}{G_s}=\sum_{\ell=1}^{L}\frac{q_{\ell,s}-\frac{\gamma_1}{L}}{1-\gamma_1}e^{\frac{\gamma_1\hat{y}_{\ell,s}}{L}}\\
&\le\sum_{\ell=1}^{L}\frac{q_{\ell,s}-\frac{\gamma_1}{L}}{1-\gamma_1}\left(1+\frac{\gamma_1\hat{y}_{\ell,s}}{L}+\left(\frac{\gamma_1\hat{y}_{\ell,s}}{L}\right)^2\right)\\
&\le 1+\frac{\gamma_1/L}{1-\gamma_1}\sum_{\ell=1}^{L}q_{\ell,s}\hat{y}_{\ell,s}+\frac{(\gamma_1/L)^2}{1-\gamma_1}\sum_{\ell=1}^{L}q_{\ell,s}\hat{y}_{\ell,s}^2.
\end{aligned}
\end{equation}
The second inequality holds due to the facts that $e^x\le 1 + x + x^2,\forall x\in [0,1]$ and $\frac{\gamma_1\hat{y}_{\ell,s}}{L}\in [0,1]$. Notice that 
\begin{equation}
\begin{aligned}
\sum_{\ell=1}^{L}q_{\ell,s}\hat{y}_{\ell,s}&=y_{\ell_s,s},\\
\sum_{\ell=1}^{L}q_{\ell,s}\hat{y}_{\ell,s}^2&=q_{\ell_s,s}\frac{y_{\ell_s,s}}{q_{\ell_s,s}}\le \hat{y}_{\ell_s,s}=\sum_{\ell=1}^{L}\hat{y}_{\ell,s}.
\end{aligned}
\end{equation}

Taking logarithms on both sides of (\ref{eq:Gt}) and summing over $s$ gives
\begin{align}
\ln\frac{G_{S+1}}{G_{1}}\le \frac{\gamma_1/L}{1-\gamma_1}\sum_{s=1}^{S}y_{\ell_s,s}+\frac{(\gamma_1/L)^2}{1-\gamma_1}\sum_{s=1}^{S}\sum_{\ell=1}^{L}\hat{y}_{\ell,s}.
\end{align}
Meanwhile, for every $\ell$,
\begin{equation}
\begin{aligned}
\ln\frac{G_{S+1}}{G_{1}}&\ge\ln\frac{g_{\ell,S+1}}{G_1}=\ln\frac{g_{\ell,1}e^{\frac{\gamma_1}{L}\sum_{s=1}^{S}\hat{y}_{\ell,s}}}{G_1}\\
&=\frac{\gamma_1}{L}\sum_{s=1}^{S}\hat{y}_{\ell,s}-\ln L.
\end{aligned}
\end{equation}
Therefore, we have
\begin{align}\label{eq:sumy}
\sum_{s=1}^{S}y_{\ell_s,s}\ge (1-\gamma_1)\sum_{s=1}^{S}\hat{y}_{\ell,s}-\frac{L\ln L}{\gamma_1}-\frac{\gamma_1}{L}\sum_{s=1}^{S}\sum_{\ell=1}^{L}\hat{y}_{\ell,s}.
\end{align}
We take expectation on both sides of (\ref{eq:sumy}) over the randomness of $y_{\ell,s}$ for all $\ell$ and $s$ (more specifically, the randomness of the arm-level EXP3 algorithm run on the $\ell$-the group within the $r$-th epoch), conditioned on the sequence of selected arm groups $(\ell_1,...,\ell_s)$ and past observations $\{y_{\ell_{\sigma},\sigma}\}_{\sigma=1}^{s}$. Note that for every fixed sequence of reward assignment, $y_{\ell,s}$ is independent across $\ell$ and $s$. Moreover, $y_{\ell,s}$ is independent of the past history of group selection, i.e., $(\ell_1,...,\ell_s)$. Therefore, we can obtain
\begin{equation}
\begin{aligned}
\sum_{s=1}^{S}x_{\ell_s,s}\ge& (1-\gamma_1)\sum_{s=1}^{S}\frac{x_{\ell,s}}{q_{\ell,s}}\mathbb{I}\{\ell_s=\ell\}-\frac{L\ln L}{\gamma_1}-\frac{\gamma_1}{L}\sum_{s=1}^{S}\sum_{\ell=1}^{L}\frac{{x}_{\ell,s}}{q_{\ell,s}}\mathbb{I}\{\ell_s=\ell\}.
\end{aligned}
\end{equation}
We further take expectation over the randomness of $(\ell_1,...,\ell_S)$ selected by the group-level EXP3 algorithm. Notice that
\begin{align}
	\mathbb{E}_{\ell_s}\left[\frac{x_{\ell,s}}{q_{\ell,s}}\mathbb{I}\{\ell_s=\ell\}\right]=\frac{x_{\ell,s}}{q_{\ell,s}}q_{\ell,s}+0\cdot(1-q_{\ell,s})=x_{\ell,s}.
\end{align}
Therefore, we have
\begin{equation}
\begin{aligned}
\mathbb{E}_{\textrm{Group-EXP3}}\left[\sum_{s=1}^{S}x_{\ell_s,s}\right]\ge (1-\gamma_1)\sum_{s=1}^{S}x_{\ell,s}-\frac{L\ln L}{\gamma_1}-\gamma_1S.
\end{aligned}
\end{equation}
Since $\ell$ is chosen arbitrarily, by choosing $\gamma_1=\sqrt{\frac{L\ln L}{2S}}$, we can conclude that
\begin{align}
\max_{1\le\ell\le  L}\sum_{s=1}^{S}x_{\ell,s}-\mathbb{E}_{\textrm{Group-EXP3}}\left[\sum_{s=1}^{S}x_{\ell_s,s}\right]\le 2\sqrt{2SL\ln L}.
\end{align}

\section{Proof of Lemma \ref{lemma:groupregret-S}}\label{apx:groupregret-S}
Let $g_{\ell,s}$ and $q_{\ell,s}$ denote the weight and the selection probability of group $\ell$ at epoch $s$. Let $G_s=\sum_{\ell=1}^{L}g_{\ell,s}$. For every $h^S=(h_1,...,h_S)$ such that $H(h^S)\le V$, consider the $V$-partition of the time horizon $[1,S]$:
\begin{align}
[S_1,...,S_2),[S_2,...,S_3),...,[S_V,...S_{V+1}),
\end{align}
where $S_1=1$ and $S_{V+1}=S+1$, such that $h_s$ is fixed for $s\in[S_v,S_{v+1}),\forall v=1,...,V$. For each segment $[S_v,S_{v+1})$:

\begin{equation}
\begin{aligned}
\frac{G_{s+1}}{G_{s}}&=\sum_{\ell=1}^{L}\frac{g_{\ell,s+1}}{G_s}=\sum_{\ell=1}^{L}\frac{g_{\ell,s}e^{\gamma_1\hat{y}_{\ell,s}/L}+\frac{e\alpha G_{s}}{L}}{G_s}\\
&=\sum_{\ell=1}^{L}\frac{q_{\ell,s}-\frac{\gamma_1}{L}}{1-\gamma_1}e^{\gamma_1\hat{y}_{\ell,s}/L}+e\alpha\\
&\le\sum_{\ell=1}^{L}\frac{q_{\ell,s}-\frac{\gamma_1}{L}}{1-\gamma_1}\left(1+\frac{\gamma_1}{L}\hat{y}_{\ell,s}+\left(\frac{\gamma_1}{L}\right)^2\hat{y}_{\ell,s}^2\right)+e\alpha\\
&\le 1+\frac{\gamma_1/L}{1-\gamma_1}\sum_{\ell=1}^{L}q_{\ell,s}\hat{y}_{\ell,s}+\frac{(\gamma_1/L)^2}{1-\gamma_1}\sum_{\ell=1}^{L}q_{\ell,s}\hat{y}_{\ell,s}^2+e\alpha.
\end{aligned}
\end{equation}
We can further derive that
\begin{equation}
\begin{aligned}
\ln\frac{G_{s+1}}{G_s}&\le \frac{\gamma_1/L}{1-\gamma_1}\sum_{\ell=1}^{L}q_{\ell,s}\hat{y}_{\ell,s}+\frac{(\gamma_1/L)^2}{1-\gamma_1}\sum_{\ell=1}^{L}q_{\ell,s}\hat{y}_{\ell,s}^2+e\alpha\\
&\le \frac{\gamma_1/L}{1-\gamma_1}y_{\ell_s,s}+\frac{(\gamma_1/L)^2}{1-\gamma_1}\sum_{\ell=1}^{L}\hat{y}_{\ell,s}+e\alpha.
\end{aligned}
\end{equation}
Summing over $s=S_v,...,S_{v+1}-1$, we have
\begin{equation}
\begin{aligned}
\ln \frac{G_{S_{v+1}}}{G_{S_v}}\le &\frac{\gamma_1/L}{1-\gamma_1}\sum_{s=S_{v}}^{S_{v+1}-1}y_{\ell_s,s}+\frac{(\gamma_1/L)^2}{1-\gamma_1}\sum_{s=S_v}^{S_{v+1}-1}\sum_{\ell=1}^{L}\hat{y}_{\ell,s}+e\alpha(S_{v+1}-S_v).
\end{aligned}
\end{equation}
By abuse of notation, we let $h_v$ be the action in this segment and then
\begin{equation}
\begin{aligned}
g_{h_v,S_{v+1}}&\ge g_{h_v,S_{v}+1}\exp\left(\frac{\gamma_1}{L}\sum_{s=S_v+1}^{S_{v+1}-1}\hat{y}_{h_v,s}\right)\\
&\ge \frac{e\alpha}{L}G_{S_{v}}\exp\left(\frac{\gamma_1}{L}\sum_{s=S_v+1}^{S_{v+1}-1}\hat{y}_{h_v,s}\right)\\
&\ge \frac{\alpha}{L}G_{S_{v}}\exp\left(\frac{\gamma_1}{L}\sum_{s=S_v}^{S_{v+1}-1}\hat{y}_{h_v,s}\right),
\end{aligned}
\end{equation}
where the last inequality holds since 
\begin{align}
\hat{y}_{h_v,s}\le1/q_{h_v,s}\le L/\gamma_1, \forall s.
\end{align}
Therefore, we have
\begin{align}
\ln \frac{G_{S_{v+1}}}{G_{S_v}}\ge\ln \left(\frac{\alpha}{L}\right)+\frac{\gamma_1}{L}\sum_{s=S_v}^{S_{v+1}-1}\hat{y}_{h_v,s},
\end{align}
and as a consequence, 
\begin{equation}
\begin{aligned}
\sum_{s=S_{v}}^{S_{v+1}-1}y_{\ell_s,s}\ge& (1-\gamma_1)\sum_{s=S_v}^{S_{v+1}-1}\hat{y}_{h_v,s}-\frac{L\ln(L/\alpha)}{\gamma_1}-\frac{\gamma_1}{L}\sum_{s=S_v}^{S_{v+1}-1}\sum_{\ell=1}^{L}\hat{y}_{\ell,s}-\frac{e\alpha L(S_{v+1}-S_v)}{\gamma_1}.
\end{aligned}
\end{equation}
We sum over all segments $v$ and take expectation on the both side of the inequality, using a similar argument as that used in the proof of Lemma \ref{lemma:groupregret}, we can obtain that
\begin{equation}
\begin{aligned}
\sum_{s=1}^{S}x_{h_s,s}-\mathbb{E}_{\textrm{Group-EXP3.S}}\left[\sum_{s=1}^{S}x_{\ell_s,s}\right]\le &\gamma_1 S + \frac{LV\ln(LS)}{\gamma_1}+\gamma_1 S +\frac{eL}{\gamma_1}
\end{aligned}
\end{equation}
if we choose $\alpha=1/S$. We further choose $\gamma_1=\sqrt{\frac{LV\ln(LS)}{S}}$ to obtain the conclusion of Lemma 3 (assuming without loss of generality that $V\ln(LS)\ge e$).

\section{Proof of Theorem \ref{thm:weakregretML}}\label{apx:weakregretML}
The proof follows the same structure with the one in the proof of Theorem \ref{thm:weakregretUB}. Let $i_{\max}$ be the arm with the greatest cumulative reward. Let $\mathcal{A}_{\ell_{\max}}$ and $\mathcal{B}_{h_{\max}}^{\ell_{\max}}$ be the group and subgroup to which $i_{\max}$ belongs. We decompose the expected weak regret of HLMC-3L as follows:
\begin{eqnarray}\nonumber
\mathbb{E}_{\textrm{HLMC-3L}}[R_{\textrm{w}}(T)]&\le & (C_{\max}-C_{\max}')+ (C_{\max}'-C_{\max}'') + (C_{\max}'' - C_{\textrm{HLMC-3L}}) \\
&=& R_1(T) + R_2(T) + R_3(T),
\end{eqnarray}
where 
\begin{eqnarray}
C_{\max}&=&\sum_{t=1}^{T}r_{i_{\max},t},\\
C_{\max}'&=&\sum_{s=1}^{S_1}\sum_{\tau=1}^{S_2}\mathbb{E}_{\textrm{Arm-EXP3}(\mathcal{B}_{h_{\max}}^{\ell_{\max}})}\left[\sum_{t\in\mathcal{I}_{\tau}^s}r_{i_t,t}\right],\\
C_{\max}''&=&\sum_{s=1}^{S_1}\mathbb{E}_{\textrm{Subgroup-EXP3}(\mathcal{A}_{\ell_{\max}})}\left[\sum_{\tau=1}^{S_2}\mathbb{E}_{\textrm{Arm-EXP3}(\mathcal{B}_{h_{\tau}}^{\ell_{\max}})}\left[\sum_{t\in\mathcal{I}_{\tau}^s}r_{i_t,t}\right]\right],\\\nonumber
C_{\textrm{HLMC-3L}}&=&\mathbb{E}_{\textrm{HLMC-3L}}\left[\sum_{t=1}^{T}r_{i_t,t}]\right],\\
&=&\mathbb{E}_{\textrm{Group-EXP3}}\left[\sum_{s=1}^{S_1}\mathbb{E}_{\textrm{Subgroup-EXP3}(\mathcal{A}_{\ell_s})}\left[\sum_{\tau=1}^{S_2}\mathbb{E}_{\textrm{Arm-EXP3}(\mathcal{B}_{h_{\tau}}^{\ell_{s}})}\left[\sum_{t\in\mathcal{I}_{\tau}^s}r_{i_t,t}\right]\right]\right].
\end{eqnarray}

Specifically, $R_1(T)$ corresponds to the arm-level reward loss due to not playing the best arm, assuming that group  $\mathcal{A}_{\ell_{\max}}$ and subgroup $\mathcal{B}_{h_{\max}}^{\ell_{\max}}$ are selected at all epochs and subepochs. By applying Lemma \ref{lemma:armregret} at every subepoch, we obtain that
\begin{equation}\label{eq:R1ml}
R_1(T)\le 2S_1S_2\sqrt{2S_3N_3\ln N_3}.
\end{equation}

For $R_1(T)$, which corresponds to the subgroup-level reward loss due to not selecting subgroup $\mathcal{B}_{h_{\max}}^{\ell_{\max}}$ at all subepochs, assuming that group $\mathcal{A}_{\ell_{\max}}$ is selected at all epochs, we apply Lemma \ref{lemma:groupregret} at every epoch by defining 
\begin{equation}
x_{h,\tau}^{\ell,s}=\mathbb{E}_{\textrm{Arm-EXP3}(\mathbb{B}_{h}^{\ell})}\left[\frac{1}{S_3}\sum_{t\in\mathcal{I}_{\tau}^s}r_{i_t,t}\right].
\end{equation}
Then we obtain that
\begin{equation}\label{eq:R2ml}
R_2(T)\le 2S_1S_3\sqrt{2S_2N_2\ln N_2}.
\end{equation}

Finally, $R_3(T)$ corresponds to the group-level reward loss due to not selecting group $\mathcal{A}_{\ell_{\max}}$ at all epochs. By defining 
\begin{equation}
z_{\ell,s}=\mathbb{E}_{\textrm{Subgroup-EXP3}(\mathbb{A}_{\ell})}\left[\frac{1}{S_2}\sum_{\tau=1}^{S_2}x_{h_\tau,\tau}^{\ell,s}\right],
\end{equation}
we can apply Lemma \ref{lemma:groupregret} again to obtain that
\begin{equation}\label{eq:R3ml}
R_3(T)\le 2S_2S_3\sqrt{2S_1N_1\ln N_1}.
\end{equation}

The upper bound in Theorem \ref{thm:weakregretML} is obtained by combining (\ref{eq:R1ml}), (\ref{eq:R2ml}), and (\ref{eq:R3ml}) together and selecting $$S_i=\left\lceil\frac{T^{1/3}(N_i\ln N_i)^{2/3}}{(\prod_{j\neq i}N_j\ln N_j)^{1/3}}\right\rceil,\forall i=1,2,3.$$

\end{document}